\PassOptionsToPackage{unicode}{hyperref}
\PassOptionsToPackage{hyphens}{url}
\PassOptionsToPackage{dvipsnames,svgnames,x11names}{xcolor}
\documentclass[
  12pt]{article}

\usepackage{amsmath,amssymb}
\usepackage{iftex}
\ifPDFTeX
  \usepackage[T1]{fontenc}
  \usepackage[utf8]{inputenc}
  \usepackage{textcomp} 
\else 
  \usepackage{unicode-math}
  \defaultfontfeatures{Scale=MatchLowercase}
  \defaultfontfeatures[\rmfamily]{Ligatures=TeX,Scale=1}
\fi
\usepackage{lmodern}
\ifPDFTeX\else  
\fi
\IfFileExists{upquote.sty}{\usepackage{upquote}}{}
\IfFileExists{microtype.sty}{
  \usepackage[]{microtype}
  \UseMicrotypeSet[protrusion]{basicmath} 
}{}
\makeatletter
\@ifundefined{KOMAClassName}{
  \IfFileExists{parskip.sty}{%
    \usepackage{parskip}
  }{
    \setlength{\parindent}{0pt}
    \setlength{\parskip}{6pt plus 2pt minus 1pt}}
}{
  \KOMAoptions{parskip=half}}
\makeatother
\usepackage{xcolor}
\makeatletter
\ifx\paragraph\undefined\else
  \let\oldparagraph\paragraph
  \renewcommand{\paragraph}{
    \@ifstar
      \xxxParagraphStar
      \xxxParagraphNoStar
  }
  \newcommand{\xxxParagraphStar}[1]{\oldparagraph*{#1}\mbox{}}
  \newcommand{\xxxParagraphNoStar}[1]{\oldparagraph{#1}\mbox{}}
\fi
\ifx\subparagraph\undefined\else
  \let\oldsubparagraph\subparagraph
  \renewcommand{\subparagraph}{
    \@ifstar
      \xxxSubParagraphStar
      \xxxSubParagraphNoStar
  }
  \newcommand{\xxxSubParagraphStar}[1]{\oldsubparagraph*{#1}\mbox{}}
  \newcommand{\xxxSubParagraphNoStar}[1]{\oldsubparagraph{#1}\mbox{}}
\fi
\makeatother

\usepackage{longtable,booktabs,array}
\usepackage{calc} 
\usepackage{etoolbox}
\makeatletter
\patchcmd\longtable{\par}{\if@noskipsec\mbox{}\fi\par}{}{}
\makeatother
\IfFileExists{footnotehyper.sty}{\usepackage{footnotehyper}}{\usepackage{footnote}}
\makesavenoteenv{longtable}
\usepackage{graphicx}
\makeatletter
\def\maxwidth{\ifdim\Gin@nat@width>\linewidth\linewidth\else\Gin@nat@width\fi}
\def\maxheight{\ifdim\Gin@nat@height>\textheight\textheight\else\Gin@nat@height\fi}
\makeatother
\setkeys{Gin}{width=\maxwidth,height=\maxheight,keepaspectratio}
\makeatletter
\def\fps@figure{htbp}
\makeatother

\makeatletter
\@ifpackageloaded{caption}{}{\usepackage{caption}}
\AtBeginDocument{%
\ifdefined\contentsname
  \renewcommand*\contentsname{Table of contents}
\else
  \newcommand\contentsname{Table of contents}
\fi
\ifdefined\listfigurename
  \renewcommand*\listfigurename{List of Figures}
\else
  \newcommand\listfigurename{List of Figures}
\fi
\ifdefined\listtablename
  \renewcommand*\listtablename{List of Tables}
\else
  \newcommand\listtablename{List of Tables}
\fi
\ifdefined\figurename
  \renewcommand*\figurename{Figure}
\else
  \newcommand\figurename{Figure}
\fi
\ifdefined\tablename
  \renewcommand*\tablename{Table}
\else
  \newcommand\tablename{Table}
\fi
}
\@ifpackageloaded{float}{}{\usepackage{float}}
\floatstyle{ruled}
\@ifundefined{c@chapter}{\newfloat{codelisting}{h}{lop}}{\newfloat{codelisting}{h}{lop}[chapter]}
\floatname{codelisting}{Listing}

\makeatother
\makeatletter
\@ifpackageloaded{caption}{}{\usepackage{caption}}
\@ifpackageloaded{subcaption}{}{\usepackage{subcaption}}
\makeatother

\ifLuaTeX
  \usepackage{selnolig}  
\fi
\usepackage[]{natbib}
\usepackage{bookmark}

\IfFileExists{xurl.sty}{\usepackage{xurl}}{} 
\hypersetup{
  pdftitle={Title},
  pdfauthor={Author 1; Author 2},
  pdfkeywords={3 to 6 keywords, that do not appear in the title},
  colorlinks=true,
  linkcolor={blue},
  filecolor={Maroon},
  citecolor={Blue},
  urlcolor={Blue},
  pdfcreator={LaTeX via pandoc}}

\usepackage{amsthm}
\usepackage{authblk}
\usepackage{comment}
\usepackage{algorithm}
\usepackage{algpseudocode}
\usepackage{setspace}

\newtheorem{lemma}{Lemma}
\newtheorem{assumption}{Assumption}
\newtheorem{proposition}{Proposition}

\DeclareMathOperator{\diff}{Diff}

\newcommand{\anon}{1}

\begin{document}

\def\spacingset#1{\renewcommand{\baselinestretch}%
{#1}\small\normalsize} \spacingset{1}


\if1\anon
{
  \title{\bf Predicting Covariate-Driven Spatial Deformation for Nonstationary Gaussian Processes}
  \author[1]{Minghao Gu}
  \author[2]{Weizhi Lin}
  \author[1]{Qiang Huang}
  \affil[1]{Daniel J. Epstein Department of Industrial \& Systems Engineering, Univeristy of Southern California}
\affil[2]{Charles W. Davidson College of Engineering, San Jos\'e State University}
  \maketitle
} \fi

\if0\anon
{
  \bigskip
  \bigskip
  \bigskip
  \begin{center}
    {\LARGE\bf Title}
\end{center}
  \medskip
} \fi

\bigskip
\begin{abstract}

Nonstationary Gaussian processes (GPs) are essential for modeling complex, locally heterogeneous spatial data. 
A common modeling approach is the spatial deformation method that warps the domain to recover isotropy. 
However, this static method does not account for changes in spatial correlation induced by covariates, limiting its ability to predict nonstationary GPs under new covariate conditions. 

To enable predictive modeling of the deformation method, we propose to model the spatial deformation as a function of covariates. 
The spaces of diffeomorphic deformations and Euclidean covariate vectors are connected by characterizing deformations as generated by velocity fields living in a Lie algebra. 
To overcome the estimation instability caused by high-order interactions between multiple covariates in a general Lie algebra, we prove that those interactions can be truncated with a moderate physical assumption. 
Based on the theoretical results, a concise functional form of deformations driven by multiple covariates can be established, and an efficient estimation-inference algorithm is developed for out-of-sample nonstationary GP prediction with limited covariate-deformation sample pairs. 
The effectiveness and generalizability of the method are demonstrated on a simulation study and two case studies, in the fields of manufacturing and geostatistics, respectively. 
\end{abstract}

\noindent%
\vfill

\newpage
\spacingset{1.8} 

\section{Introduction}\label{sec_intro}

Gaussian process (GP) models provide a flexible, data-efficient approach for analyzing spatial functional data, with wide-ranging applications across scientific and engineering domains, such as computer experiment, geostatistics, and machine learning \citep{kennedy2001bayesian, stein1999interpolation, seeger2004gaussian}. 
Among the first- and second-order properties fully defining a GP, 
it is standard practice to assume zero-mean and focus on covariance modeling, as deterministic trends can be modeled and subtracted a priori. 
In its classical form, a stationary covariance structure is adopted, encoding assumptions of homogeneity and translational invariance over the spatial domain. 
While such assumptions are commonly used to facilitate tractable modeling and inference, they are often violated in real-world applications, where the stochastic behavior of a process can vary across locations. 

\begin{figure}[!htp]
    \centering
    \includegraphics[width=0.9\linewidth]{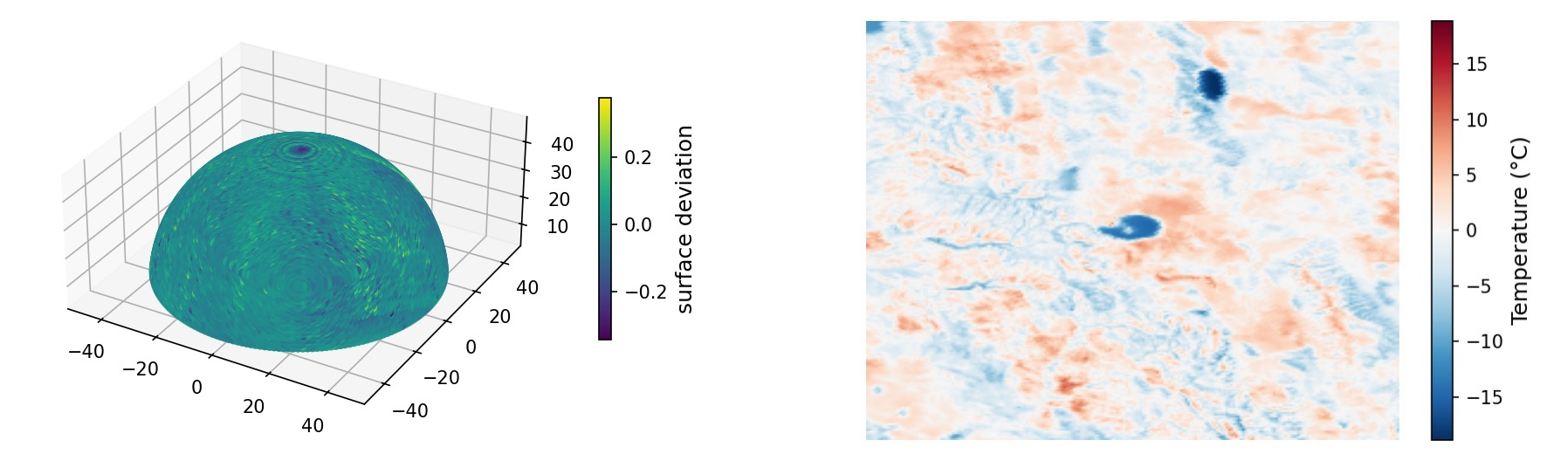}
    \caption{\label{fig_example}Left: surface deviation measured on a dome shape AM part, exhibiting heterogeneous local patterns. Right: the temperature residual over a complex mountainous region, which shows significant nonstationarity.}
\end{figure}


We present two datasets from distinct systems in Fig.~\ref{fig_example}: additive manufacturing (AM) surface quality and geographical terrain temperature. 
The left plot displays the surface deviation measured on a dome-shaped part fabricated by a fused deposition modeling (FDM) process, which exhibits heterogeneous patterns locally. 
For instance, the apex demonstrates a pronounced staircase effect \citep{matos2020improving}, whereas the bottom shows stronger vertical spatial correlation. 
The right plot illustrates the temperature residual, after extracting the deterministic mean \citep{hijmans2005very}, over a complex mountainous region \citep{wan2014new, farr2007shuttle}. 
From the plot, this zero-mean residual process retains significant nonstationarity. 
Furthermore, we can discover that the nonstationarity in both systems is modulated by certain local conditions: \cite{gu2026surface} noted that the heterogeneous patterns of surface deviation are jointly influenced by local surface geometry and process-related factors, while the temperature is clearly related to local topography. 
These experimental or environmental conditions impacting the nonstationarity are considered as covariates in these systems. 
Consequently, classical stationary models are inadequate, necessitating methodologies that can explicitly link heterogeneous spatial covariance to local covariates. 


\begin{figure}[!htp]
    \centering
    \includegraphics[width=0.9\linewidth]{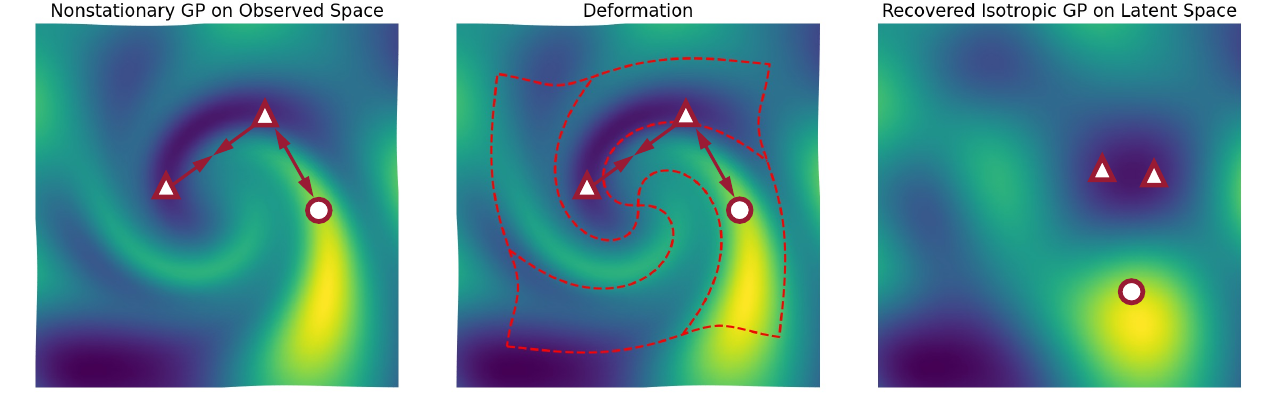}
    \caption{\label{fig_deform_illustration}Illustration of the deformation method idea. For the original observed space (left), strongly correlated locations (triangles) are compressed together, and weakly correlated locations (triangle and circle) are stretched apart. Applying the continuous deformation (middle) to the original space, the GP on the deformed latent space is recovered to isotropy (right).}
\end{figure}

To model such complex spatial nonstationarity, various strategies have been proposed, including nonstationary kernels and piece-wise stationary models \citep{fouedjio2017second, sauer2023non}.
Among these approaches, the spatial deformation method originally introduced by \cite{sampson1992nonparametric} stands out, benefiting from two unique advantages: 
As illustrated by Fig.~\ref{fig_deform_illustration}, the method treats the observed spatial domain (left) as pliable. 
Locations that exhibit unexpectedly strong correlation (triangles) are pulled closer together, while locations with unusually weak correlation (triangle and circle) are pushed apart. 
Through compressing and stretching the original coordinates (middle), the complex spatial correlation is simplified into a deformed isotropic one in the latent space (right) that can be handled with classic tools. 
Compared to other approaches, this method benefits from two unique advantages: 
\begin{itemize}
    \item Geometric interpretation: The deformation mapping ("compressing and stretching") provides a highly intuitive geometric interpretation of the system's underlying heterogeneity. 
    \item Guaranteed validity: Unlike methods that require complex regularization to maintain mathematically valid, the deformation method inherently guarantees the positive-definiteness of the resulting covariance structure. 
\end{itemize}
This deformation mapping ("compressing and stretching") is typically represented via basis expansion \citep{sampson1992nonparametric, perrin1998modeling} or Beltrami coefficients \citep{anderes2008estimating}, encoding the invertible mapping between the observed and latent domain. 

However, while the deformation method provides an advantageous foundation for modeling nonstationarity, the existing formulations were primarily designed to solve interpolation problems under sparse observation settings. 
Spatial interpolation, or kriging, is an in-sample task that estimates values at unobserved locations within a fixed spatial domain under static conditions. 
This interpolation is typically most meaningful when data are acquired at low spatial resolution.

However, the development of high-resolution data acquisition technologies nowadays, such as optical surface scanning and advanced remote sensing, has fundamentally shifted the engineering bottleneck. 
With denser observations increasingly accessible, the contemporary challenge is no longer interpolating fixed domains under static conditions, but predicting across unknown ones. 
Returning to our examples in Fig.~\ref{fig_example}, the potential objective is no longer interpolating the scanned deviations of an existing AM dome, but rather predicting the nonstationary surface quality of a newly designed component prior to fabrication, or similarly, forecasting the temperature field of another unmonitored mountain. 
To establish such out-of-sample prediction capability with the deformation method, a model must be able to extrapolate how a new spatial domain will deform. 
Because the underlying system's nonstationarity is often governed by local covariates, as demonstrated in our examples, these new, extrapolative spatial deformations must naturally be modeled as functions of those local covariates. 

In contrast to this practical need, a research gap persists: there is a lack of methodology for modeling and predicting these covariate-driven deformations. 
Developing such a method poses two unique challenges. 
First, deformations and covariates exist in spaces with fundamentally distinct topologies: A valid spatial deformation is a highly structured mapping, usually a diffeomorphism, residing in a nonlinear, curved space, whereas covariates typically reside in standard Euclidean space. 
Establishing a mapping between these distinct topologies is inherently difficult, since standard Euclidean regression cannot be applied. 
Second, the number of observations under different covariate conditions available for training is typically limited due to the time consumption and cost of experimentation. 
Consequently, if we adopt heavily parameterized or black-box mapping approaches for capturing the complex mapping between spatial deformations and covariates, they lack the structural constraints needed for reliable extrapolation and are prone to severe overfitting. 

To enable out-of-sample prediction of nonstationary GP, we propose a novel covariate-driven deformation modeling method. 
We first characterize the spatial deformation via the Lie algebra.
By projecting the complex space of valid spatial deformations to its Lie algebra, i.e., a vector space of velocity fields, it provides the necessary tool to explicitly link Euclidean covariates to spatial deformations, resolving the topological mismatch. 
To further address the highly complex computational rules between multiple covariate channels, we formulate the covariate-driven velocity field based on principles of covariate-channel-wise commutativity. 
This commutativity not only bypasses the computationally intensive Lie algebra integrations but also enforces a concise model that remains effective and robust under limited sample sizes. 
We present a series of experimental studies to verify the out-of-sample prediction capability of the covariate-driven model as well as its broad generalizability across diverse physical domains. 

The remainder of the paper is organized as follows. 
Sec.~\ref{sec_review} gives a thorough literature review on existing approaches to modeling nonstationary GPs, revealing the research gap in predicitve modeling with the deformation method. 
Sec.~\ref{sec_method} proposes our covariate-driven deformation modeling method. 
A simulation study and two case studies are presented in Sec.~\ref{sec_experiment} to verify the effectiveness of our method. 
Sec.~\ref{sec_conclusion} summarizes our work and gives potential future research directions. 

\section{Literature Review: Nonstationary GPs} \label{sec_review} 


 
To model nonstationary GPs, multiple strategies have been proposed, falling into three main categories \citep{sauer2023non, fouedjio2017second}: nonstationary kernels, divide-and-conquer, and spatial deformation. 
In this section, we review these categories, with an emphasis on the deformation method and covariate-driven subsets \citep{risser2016nonstationary}. 

The most direct approach for modeling nonstationarity is to encode it analytically into the covariance kernel. 
Let $\Sigma(\cdot)$ denote a stationary GP kernel; a nonstationary extension replaces the distance-based form with a location-dependent kernel $\Sigma_{\mathrm{ns}}(\boldsymbol{x}_i,\boldsymbol{x}_j)$ that depends explicitly on locations $\boldsymbol{x}_i$ and $\boldsymbol{x}_j$. 
A general construction is given by process convolutions $\Sigma_{\text{ns}}(\boldsymbol{x}_i, \boldsymbol{x}_j) = \int \kappa_{\boldsymbol{x}_i}(u) \kappa_{\boldsymbol{x}_j}(u) du$, where the local kernels $\kappa_{\boldsymbol{x}}$ vary with spatial location, thereby inducing spatially adaptive dependence \citep{higdon1999non, paciorek2003nonstationary, nychka2002multiresolution}. 
Another more explicit route is to parameterize kernel hyperparameters, e.g., lengthscale and variance, as functions of location covariates \citep{heinonen2016non, binois2018practical}. 
This branch is particularly relevant for covariate-dependent modeling \citep{reich2011class, neto2014accounting}, as further integrating covariates into the hyperparameters is straightforward in form. 
These models provide high flexibility and interpretability by directly linking spatial variation to physical conditions, but they come with multiple mathematical hurdles. 
As local kernel parameters must be estimated across spatial and covariate domains, these models are heavily parameterized. 
While approaches like SPDE-based representations \citep{lindgren2011explicit, ingebrigtsen2014spatial} and multi-resolution approximations \citep{nychka2002multiresolution, katzfuss2017multi} mitigate computational burden via sparsity, they do not resolve two underlying issues. 
Under limited sample sizes, estimating these vast parameters becomes an ill-posed problem. 
Furthermore, the attempt to force nonstationarity by manipulating the algebraic structure of the covariance function makes it numerically vulnerable.
Injecting covariates exacerbates this problem, as abrupt fluctuations of covariates will break the smooth distance metric \citep{bornn2012modeling}, on which a valid spatial covariance fundamentally relies \citep{fuglstad2015does}. 
Consequently, ensuring that the resulting covariance matrix remains positive-definite requires restrictive regularization \citep{paciorek2006spatial, fuglstad2015does}.

Conversely, divide-and-conquer methods \citep{rasmussen2001infinite, kim2005analyzing} assume piecewise stationarity and partition the spatial domain into smaller, tractable sub-regions, thereby bypassing the problem of regularization. 
One milestone work is the treed GP proposed in \cite{gramacy2008bayesian}, which utilizes Bayesian regression trees to generate spatial partitions, allowing efficient recursive cuts of high-dimensional domains. 
\cite{gramacy2015local} further advanced the method for large-scale computer experiments via local approximations. 
Notably, these divide-and-conquer models can use covariates as partitioning criteria to isolate sub-regions with distinct behavior \citep{konomi2014bayesian}, or as gating functions to determine local Gaussian experts \citep{tresp2000mixtures}. 
These models operate under the assumption of piecewise stationarity, fitting independent, localized stationary surrogates within each discrete sub-region. 
While highly computationally efficient, this means these models inevitably suffer from numerical discontinuity between sub-regions.
Because the local models are estimated independently, the resulting global field lacks the smoothness and coherence required to model continuous phenomena like thermal gradients or manufactured surfaces. 

The search for a nonstationary GP model that avoids heavy algebraic constraints on the covariance structure, remains interpretable, yet accommodates globally continuous variation naturally leads us to the spatial deformation method. 
Originally proposed by \cite{sampson1992nonparametric}, this approach reframes the problem: rather than revising the covariance kernel itself, the deformation method instead focuses on learning a bijective spatial mapping, often a diffeomorphism, from the original observed spatial coordinates to a latent space. 
Therefore, the method transfers the regularization burden from the covariance structure to the spatial deformation, inherently preserving the positive definiteness of the covariance as long as the deformation is bijective. 
This is an advantageous trade, as there are more robust tools for enforcing deformation bijectivity than for stabilizing covariance matrices. 
\cite{perrin1998modeling} uses bijective bases to expand the deformation, while \cite{iovleff2004estimating} meshes the domain and regulates the orientation of the triangular mesh. 
\cite{damian2001bayesian} and \cite{schmidt2003bayesian} estimate the deformation within a Bayesian paradigm, mitigating folding with a Bayesian prior. 
\cite{anderes2008estimating} utilizes the differential geometry formulation to not only regulate the deformation but also address the estimation with a single realization. 
The fundamental philosophy of spatial deformation has later enlightened deep GP \citep{damianou2013deep, dunlop2018deep}, in which neural network (NN) layers act as sequential, data-driven spatial deformations. 
However, despite these profound advancements, the spatial deformation literature remains constrained: Whether utilizing basis expansion or deep NN, existing deformation methods are designed to learn a static deformation tailored to the observed data. 
As a result, while they provide coherent and accurate in-sample modeling, they do not address how the deformation should adapt when system dynamics change with covariates. 
In particular, the model offers little guidance on answering the out-of-sample question: How the latent space will deform under novel covariate conditions. 

Synthesizing the review across these three paradigms for modeling nonstationary GPs reveals a clear research gap: covariate-dependent prediction via nonstationary kernels or discrete partitioning inevitably compromises numerical stability or global continuity, while conversely, the spatial deformation method remains static and decoupled from physical drivers of nonstationarity. 
Therefore, this is a crucial need for a method that bridges the divide. 


\section{Methodology} \label{sec_method}

\subsection{Preliminary, Notation, and Setup on Deformation Method} \label{sec_setup}

Throughout this work, we use plain lowercases (e.g., $x$) for scalars, bold lowercases ($\boldsymbol{x}$) for vectors, plain uppercases ($X$) for functions or stochastic processes, bold uppercases ($\boldsymbol{X}$) for matrices, and caligraphic letters ($\mathcal{X}$) for domains or sets. 

Let $Y(\cdot)$ denote a continuous process observed over a smooth, compact manifold $\mathcal{S}$, with coordinates $\boldsymbol{s} \in \mathcal{S}$. 
Without loss of generality, we model $Y(\cdot)$ as a mean-zero GP, with its covariance function denoted as $\Sigma(\cdot, \cdot)$. 
The covariance does violate second-order stationarity, i.e., $\Sigma(\boldsymbol{s}_i, \boldsymbol{s}_j) \neq \Sigma \left( \Vert \boldsymbol{s}_i - \boldsymbol{s}_j \Vert \right)$. 

In the deformation method \citep{sampson1992nonparametric}, the domain $\mathcal{S}$ of the nonstationary process $Y(\cdot)$ goes through a spatial deformation $f: \mathcal{S} \rightarrow \mathcal{W}$ and is deformed into a new, latent domain $\mathcal{W}$. 
For simplicity, we consider the common case that the latent domain $\mathcal{W}$ has the same topology (i.e., homeomorphic) as the observed domain $\mathcal{S}$. 
In this latent domain $\mathcal{W}$, a new spatial process $Z \left( f(\cdot) \right)$ is therefore constructed, with spatial coordinates $\boldsymbol{w} = f(\boldsymbol{s}) \in \mathcal{W}$.
This new process $Z(\cdot)$ defined on the latent domain $\mathcal{W}$ is strictly isotropic and called the base process. 
The explicit deformation relationship between the two processes can be written as: 
$$
\begin{aligned}
    Z(\boldsymbol{w}) = Z \left( f(\boldsymbol{s}) \right) = Y(\boldsymbol{s}), \\
    Y = Z \circ f, \quad Z = Y \circ f^{-1}. 
\end{aligned}
$$

Correspondingly, the nonstationary covariance on the original domain $\mathcal{S}$ can be computed as: 
$$
\Sigma(\boldsymbol{s}_i, \boldsymbol{s}_j) = C_{\boldsymbol{\theta}} \left( \Vert f(\boldsymbol{s}_i) - f(\boldsymbol{s}_j) \Vert \right) = C_{\boldsymbol{\theta}} \left( \Vert \boldsymbol{w}_i - \boldsymbol{w}_j \Vert \right),
$$
where $C_{\boldsymbol{\theta}}$ is a standard, isotropic covariance function (such as the Mat\'ern family) of the base process, governed by hyperparameters $\boldsymbol{\theta}$. 

It is obvious that the positive definiteness property of $C_{\boldsymbol{\theta}}(\cdot)$ will be kept in $\Sigma(\cdot, \cdot)$ as long as the deformation mapping $f(\cdot)$ is smooth and bijective \citep{perrin1999identifiability}. 
Therefore, the following constraints are commonly imposed on the deformation mapping as well as the base process \citep{perrin1998modeling, anderes2008estimating}: 
\begin{enumerate}
    \item $f$ is at least a $\mathcal{C}^1$-diffeomorphism onto its image. 
    \item $f$ is orientation-preserving. 
    \item $Z$ is isotropic and unit-ranged. 
\end{enumerate}
While the first constraint is for validity, the second and third constraints are for identifiability. 
$f$ being orientation-preserving eliminates the possibility of mirroring in the deformation. 
We define an isotropic GP with covariance function $\Sigma_{\text{iso}}(\cdot)$ to be unit-ranged if it satisfies $\frac{\Sigma_{\text{iso}}(1)}{\Sigma_{\text{iso}}(0)} = e^{-1}$, thereby without scaling ambiguity. 
After ruling out mirroring and scaling ambiguity, the deformation and the base process are now uniquely defined and identifiable from a single realization of $Y(\cdot)$ \citep{anderes2008estimating, fouedjio2015estimation}. 

As established in the literature review Sec.~\ref{sec_review}, the critical barrier is to utilize this deformation method formulation for out-of-sample prediction. 
While the classical deformations $f(\boldsymbol{s})$ depend exclusively on spatial coordinates $\boldsymbol{s}$, the prerequisite for achieving out-of-sample GP predictions under novel conditions is to predict the deformation as a function of the covariates as well, yielding $f(\boldsymbol{s}, \boldsymbol{\tau})$, where $\boldsymbol{\tau}$ denotes covariates. 

We consider a dataset $\mathcal{Y} = \{ Y_1, \cdots, Y_N \}$, where each $Y_k(\cdot)$ represents a nonstationary GP observed over the shared reference manifold $\mathcal{S}$. 
This is without loss of generality, because the standardized reference $\mathcal{S}$ can be found and anchored across disparate experimental fields, and this will be explained in detail in Sec~\ref{sec_experiment}. 
Each observation $k$ is associated with a specific, known covariate vector $\boldsymbol{\tau}_k(\cdot) \in \mathcal{T} \subset \mathbb{R}^p$, representing $p$ experiment/environment-related conditions, which may vary across the spatial domain $\mathcal{S}$ (for simplicity, we write $\boldsymbol{\tau}_k$ without the location indexing in what follows). 
We treat the input vector $\boldsymbol{\tau}_k$ as a set of distinct influencing channels. 
That is, the effect of each $\tau_k^m, m = 1, \cdots, p$ is independent. 
This isn't simplifying the problem too much, as interactions between covariates, such as coupling, can be addressed a priori via standard feature engineering procedures. 
To make the dataset tractable via deformation-method-based prediction, we impose the following assumption on the data generation mechanism: 
\begin{assumption}[Shared Base Process] \label{assumption_data_generation}
    For the $N$ observed nonstationary GPs in set $\mathcal{Y} = \{ Y_k \}$ all defined on $\mathcal{S}$, there exists a shared, isotropic base process $Z(\cdot)$ defined on a latent space $\mathcal{W}$ such that observed processes can be represented as covariate-driven deformations of this base process. Specifically, each $Y_k$ can be represented as: 
    $$
    Y_k  = Z \circ f_k,
    $$
    where the spatial deformation $f_k(\cdot) = f(\cdot, \boldsymbol{\tau}_k)$ is covariate-driven. 
\end{assumption}

The validity of this assumption is rigorously tested in \cite{gu2025identification} with actual manufacturing examples. 
The focus of the subsequent sub-sections is, based on this setup, to establish such a functional relationship $f \left( \cdot, \boldsymbol{\tau}_k \right)$. 

\subsection{Velocity Field Characterization of Spatial Deformations via Lie Algebra} 

Before we explicitly give the functional relationship $f \left( \cdot, \boldsymbol{\tau}_k \right)$, the fundamental challenge of topological mismatch mentioned in Sec.~\ref{sec_intro} needs to be resolved. 
Specifically, covariates $\boldsymbol{\tau}_k \in \mathcal{T} \subset \mathbb{R}^p$ exist in standard Euclidean space. 
In contrast, the valid spatial deformation $f_k$, regulated by the constraints proposed in Sec.~\ref{sec_setup}, belongs to a space of diffeomorphisms. 

We first formalize the algebraic space in which these deformations operate. 
To do this, we select one of the observed processes and its associated deformation as our baseline, anchoring the deformation space. 
Without loss of generality, we select process $Y_0(\cdot) = Y_1(\cdot)$ and associated deformation $f_0(\cdot) = f_1(\cdot) = f ( \cdot, \boldsymbol{\tau}_1 )$ as the anchoring baseline. 
From this baseline, we define 
$$
\begin{aligned}
    h_k(\boldsymbol{w}) & = h \left( \boldsymbol{w}, \Delta \boldsymbol{\tau}_k \right) = h \left( \boldsymbol{w}, \boldsymbol{\tau}_k - \boldsymbol{\tau}_0 \right) \\
    & = f \left( f_0^{-1}(\boldsymbol{w}), \boldsymbol{\tau}_k \right),
\end{aligned}
$$
where $h_{\Delta \boldsymbol{\tau}_k}: \mathcal{W} \rightarrow \mathcal{W}$ is the relative deformation shifted from the baseline $f_0$, induced by shifting the covariates $\Delta \boldsymbol{\tau}_k = \boldsymbol{\tau}_k - \boldsymbol{\tau}_0$ from the baseline state. 

As a result, these valid relative deformations are diffeomorphisms of $\mathcal{W}$ to itself, therefore forming an algebraic group denoted by $\diff(\mathcal{W})$. 
This group is closed under the composition ("$\circ$") operation, i.e., $h_k \circ h_t \in \diff(\mathcal{W})$. 
This closure guarantees that all covariate-induced relative deformations can be safely applied to the base process to yield a valid diffeomorphism, formulated as: 
$$
    f_k = h_k \circ f_0.
$$

As $\mathcal{W}$ is set to be a smooth, compact manifold (since it is diffeomorphic to $\mathcal{S}$), the group $\diff(\mathcal{W})$ is further an infinite-dimensional Fr\'echet Lie group \citep{hamilton1979inverse}. 
This points a direction: the deformations in the Lie group can be projected to the tangent space at the identity, namely the Lie algebra $\mathfrak{g}$. 
Specifically, each valid relative deformation $h_k \in \diff(\mathcal{W})$ can be projected to a vector field $V_k \in \mathfrak{g}$ through a logarithm mapping, and through an exponential mapping vice versa: $h_k(\boldsymbol{w}) = (\exp V_k)(\boldsymbol{w})$. 

Physically, $V_k$ can be interpreted as a velocity field (This terminology is commonly used in dynamical systems \citep{brin2002introduction, johnson2016handbook}, and we will borrow it in what follows.) guiding a hypothetical particle, and integrating the vectors gives the particle's trajectory. 
The set of trajectories of all particles over the domain gives the spatial deformation governed by the velocity field. 
Notably, $V_k$ allows additions and scalar multiplications, and therefore, there is no longer a topological mismatch between $\mathfrak{g}$ and the covariate space $\mathcal{T}$. 
Similar characterization of deformations is also utilized in computational anatomy (e.g., large deformation diffeomorphic metric mapping \citep{beg2005computing}). 

\subsection{Covariate-Driven Modeling of Spatial Deformations} 

Based on the velocity field characterization provided above, if each covariate channel $\tau^m_k, m = 1, \cdots, p$ induces a corresponding velocity field component $V_k^m \in \mathfrak{g}$ (channel independence is assumed in Sec.~\ref{sec_setup}), the specific velocity field induced by a covariate shift $\Delta \tau_k^m$ can be formulated clearly as a linear scaling $\Delta \tau_k^m V_k^m$. 
This is natural, since $\mathfrak{g}$, as a linear vector space, accommodates linear combination and scaling. 
The resulting spatial deformation driven by this $m$-th isolated channel is exactly:
\begin{equation}\label{eq_linear}
    h_k^m = \exp (\Delta \tau_k^m V_m). 
\end{equation}

In physical experiments, the effects of independent covariate channels act upon the system cumulatively, meaning their resulting spatial deformations should be generated by the exponential map of the linear sum of their corresponding velocity fields: 
\begin{equation} \label{eq_composition}
    h_k = \exp \left( \sum_{m = 1}^p \Delta \tau_k^m V_m \right), 
\end{equation}
where the velocity fields $V_m$ are the core objectives to be estimated. 

However, this formulation presents a severe estimation bottleneck due to entanglement between covariate channels. 
Because in general Lie algebra, arbitrary velocity fields do not commute, i.e., $[V_m, V_n] \neq 0$, where $[\cdot, \cdot]$ denotes the Lie bracket, the Baker-Campbell-Hausdorff (BCH) formula indicates that the total deformation cannot be decoupled.
\begin{equation}\label{eq_bch}
    \exp (X + Y)
    = \exp(X) \circ \exp(Y) \circ \exp \left( -\frac{1}{2}[X, Y] \right) \circ \exp \left( \frac{1}{3} [Y, [X, Y]] + \frac{1}{6} [X, [X, Y]] \right) \circ \cdots
\end{equation}

From an optimization perspective, if the total deformation is fully coupled, a gradient update to the parameters of one velocity field $V_m$ propagates nonlinearly through the nested integrations, inducing chaotic, cascading interaction effects on the final spatial mapping. 
This extreme parameter sensitivity creates a highly ill-conditioned optimization landscape, rendering gradient-based estimation highly unstable.

To salvage this formulation and recover a mathematically tractable multi-covariate model, we impose the following assumption. 

\begin{assumption}[DoE Path Independence of Covariates] \label{assumption_path_independence}
    In a multi-covariate experimental environment, the final state of the system is invariant to the sequence in which independent covariate interventions are applied. Specifically, in spatial deformation driven by covariates, let $h_k^m$ and $h_k^n$ denote the relative spatial deformations induced by shifting isolated covariate channels $m$ and $n$ by $\Delta \tau_k^m$ and $\Delta \tau_k^m$, respectively. These interventional deformations commute under composition: 
    $$
        h_k^m \circ h_k^n = h_k^n \circ h_k^m.
    $$
\end{assumption}

In the established Design of Experiments (DoE) principles \citep{fisher1966design}, when dealing with multiple control variables, the final state of a system is generally assumed to be invariant to the sequence of interventions. 
For example, in an AM process, adjusting the extruder temperature and then the printing speed should ideally yield the same product quality as speed first and temperature second, and there is no carryover effect in between. 
This implicit modeling assumption is widely adopted in experimental studies, and it directly leads to Assumption~\ref{assumption_path_independence} if the deformation induced by an isolated covariate channel is viewed as the effect of a control variable. 

This reasonable assumption of macroscopic imposes a useful geometric constraint on the microscopic velocity fields, illustrated by the following lemma: 
\begin{lemma}[Channel Commutativity of Multiple Covariates]
    \label{lemma_commute}
    Let $V_1 \cdots, V_p$ be velocity fields on $\mathcal{W}$ associated with the covariate channels $\tau^1, \cdots, \tau^p$, respectively. 
    The DoE path independence condition (Assumption~\ref{assumption_path_independence}) holds if and only if these velocity fields commute, i.e., their Lie bracket vanishes: 
    $$
        [V_m, V_n] = 0, ~ \forall m, n \in \{ 1, \cdots, p \}, ~ m \neq n. 
    $$
\end{lemma}
The formal proof of Lemma~\ref{lemma_commute} is provided in Appendix~\ref{appendix_1}. 

Based on this lemma, the estimation bottleneck identified on \ref{eq_composition} is immediately resolved, as all higher-order nested Lie brackets in the BCH formula can be truncated. 
Now, we can formalize the functional model of spatial deformations driven by multiple covariates as the following proposition: 
\begin{proposition}[Explicit Form of Covariate-Driven Deformation, Linear]
    \label{proposition_linear}
    Given a baseline spatial deformation $f_0: \mathcal{S} \rightarrow \mathcal{W}$ associated with a vector composed of independent, spatial varying covariate channels $\boldsymbol{\tau}_0 = [\tau_0^1(\boldsymbol{w}), \cdots, \tau_0^p(\boldsymbol{w})]^\top$, and another shifted covariate vector $\boldsymbol{\tau}_k$, let $V_1, \cdots, V_p$ be commutative base velocity fields corresponding to each covariate channel, in the Lie algebra $\mathfrak{g}$ of $\diff(\mathcal{W})$. The total spatial deformation, $f_k: \mathcal{S} \rightarrow \mathcal{W}$, driven by $\boldsymbol{\tau}_k$, is given by the exact sequential composition: 
    $$
    \begin{aligned}
        f_k(\boldsymbol{s}) & = h_k^p \circ \cdots \circ h_k^1 \circ f_0 (\boldsymbol{s}) \\
        & = \exp (\Delta \tau_k^p V_p) \circ \cdots \circ \exp (\Delta \tau_k^1 V_1) \circ f_0 (\boldsymbol{s}) \\
        & = \left[ \bigcirc_{m = 1}^p \exp(\Delta \tau_k^m V_m) \right] \circ f_0(\boldsymbol{s}),
    \end{aligned}
    $$
    where $\Delta \tau_k^m = \tau_k^m - \tau_0^m$ is the amount of the $m$-th covariate shifted from the baseline state. 
\end{proposition}
The formal proof of Proposition~\ref{proposition_linear} is provided in Appendix~\ref{appendix_2}. 

Proposition 1 establishes an explicit functional form of spatial deformations $f_k(\cdot) = f(\cdot, \boldsymbol{\tau}_k)$, which is precisely the goal we set in Sec.~\ref{sec_setup}. 
Based on this result, the predictive modeling of nonstationary GPs is essentially transformed into a simple regression problem, which can be solved with the algorithm discussed in Sec.~\ref{sec_estimation}. 

However, assuming the spatial deformation scales proportionally as in Proposition~\ref{proposition_linear} can be too restrictive for general experimental settings, as physical systems frequently exhibit nonlinear responses. 
To break the linear structure, we can generalize the model by introducing a link function $g_m(\cdot)$ for each channel. 
This function transforms the raw covariate input into an actual effective amount applied to the velocity field, yielding a more generalized relative mapping isolated for the $m$-th channel: 
\begin{equation}\label{eq_glm}
    h_k^m(\boldsymbol{w}) = \exp \left( g_m(\Delta \tau_k^m) V_m \right)(\boldsymbol{w}),
\end{equation}
where $g_m: \mathbb{R} \rightarrow \mathbb{R}$ are continuous, strictly monotonic functions. 

This is a direction extension of Eq.~\ref{eq_linear}, significantly enhancing the expressive power of the covariate-driven deformation model. 
To preserve the mathematical validity of the spatial deformation, these link functions are required to be continuous and monotonic, stemming from both identifiability and bijectivity. 
First, if $g_m(\cdot)$ were non-monotonic, distinct covariate shifts could produce the exact same effective flow amount, rendering the model parameters unidentifiable from the observed data. 
Second, as the covariate $\tau^m(\boldsymbol{w})$ could vary continuously across the spatial domain, an oscillating $g_m(\cdot)$ could cause adjacent spatial locations to experience reversed deformation directions. 
This can lead to spatial folding or tearing, thereby violating the bijective constraint of $f(\cdot)$. 

Eq.~\ref{eq_glm} further leads to an extension of Proposition~\ref{proposition_linear}: 
\begin{proposition}[Explicit Form of Covariate-Driven Deformation, General Monotonic]
    \label{proposition_glm}
    Besides the settings of Proposition~\ref{proposition_linear}, if the covariates are Lipschitz continuous with constants $L_m = \sup_{\boldsymbol{w} \in \mathcal{W}} \Vert \nabla \tau^m(\boldsymbol{w}) \Vert_2$, respectively, then with a set of continuous, strictly monotonic linking functions $g_1, \cdots, g_p$ satisfying $\vert g_m' \vert < \frac{1}{L_m}$, the total spatial deformation, $f_k: \mathcal{S} \rightarrow \mathcal{W}$, driven by $\boldsymbol{\tau}_k$, is given by the exact sequential composition: 
    $$
    \begin{aligned}
        f_k(\boldsymbol{s}) & = h_k^p \circ \cdots \circ h_k^1 \circ f_0 (\boldsymbol{s}) \\
        & = \exp \left( g_p \left( \Delta \tau_k^p \right) V_p \right) \circ \cdots \circ \exp \left( g_1 \left( \Delta \tau_k^1 \right) V_1 \right) \circ f_0 (\boldsymbol{s}) \\
        & = \left[ \bigcirc_{m = 1}^p \exp \left( g_m \left( \Delta \tau_k^m \right) V_m \right) \right] \circ f_0(\boldsymbol{s}),
    \end{aligned}
    $$
    where $\Delta \tau_k^m = \tau_k^m - \tau_0^m$ is the amount of the $m$-th covariate shifted from the baseline state. 
\end{proposition}
The formal proof of Proposition~\ref{proposition_glm} is provided in Appendix~\ref{appendix_3}

\subsection{Predictive Modeling of Covariate-Drive Nonstationary Gaussian Processes based on Deformation Method} \label{sec_estimation}

With the explicit functional forms established in Propositions~\ref{proposition_linear} and \ref{proposition_glm}, the complex problem of out-of-sample prediction for nonstationary GPs is reduced to learning the base velocity fields and their corresponding link functions. 
This section outlines the algorithm to utilize it for predicting nonstationary GP covariance under novel, unobserved conditions, given the setup in Sec.~\ref{sec_setup}. 
This end-to-end algorithm is formalized in Algorithm~\ref{alg_pipeline}. 

\begin{algorithm}[htbp]
\caption{Predictive Modeling of Covariate-Driven Nonstationary GPs}
\label{alg_pipeline}
\begin{spacing}{0.85}
\begin{algorithmic}[1]
\Require 
    A dataset $\mathcal{Y} = \{ Y_k \}_{k = 1}^N$ of nonstationary GPs and associated covariate vectors $\boldsymbol{\tau}_k$ defined on domain $\mathcal{S}$; 
    estimated deformations $f_k$ mapping them to the same base process $Z$ defined on $\mathcal{W}$; 
    selected baseline spatial mapping $f_0$ and associated baseline covariates $\boldsymbol{\tau}_0$; 
    novel covariates $\tau_{\text{new}}$ for prediction. 
\Ensure 
    Predicted nonstationary GP covariance $\hat{\Sigma}_{\text{new}}$.
\vspace{1.5mm}
\Statex \textbf{\textit{Phase 1: Estimation (Training)}}
\vspace{1.5mm}
\State \textbf{Initialize:} Velocity fields $V_m$ and link functions $g_m$ for $m=1,\dots,p$.
\While{not converged}
    \State Loss $\mathcal{L}_{total} \leftarrow 0$
    \For{$k = 1$ to $N$}
        \State \textbf{Initialize:} $\hat{f}_k(\boldsymbol{s}) \leftarrow f_0(\boldsymbol{s}), \quad \forall \boldsymbol{s} \in \mathcal{S}$
        \For{$m = 1$ to $p$}
            \State Calculate covariate shift: $\Delta\tau_k^m(\boldsymbol{s}) = \tau_k^m(\boldsymbol{s}) - \tau_0^m(\boldsymbol{s})$
            \State Compute effective flow amount: $t_m(\boldsymbol{s}) = g_m(\Delta\tau_k^m(\boldsymbol{s}))$
            \State Integrate via ODE solver: $\hat{f}_k(\boldsymbol{s}) \leftarrow \hat{f}_k(\boldsymbol{s}) + \exp \left( t_m(\boldsymbol{s}) V_m \right)$
        \EndFor
        \State Accumulate mapping error: $\mathcal{L}_{total} \leftarrow \mathcal{L}_{total} + \text{Loss}(\hat{f}_k(\boldsymbol{s}), f_k(\boldsymbol{s}))$
    \EndFor
    \State Update parameters $V_m, g_m$ via gradient descent to minimize $\mathcal{L}_{total}$.
\EndWhile
\State \textbf{Fix parameters:} $V_m \leftarrow V_m^*$, $g_m \leftarrow g_m^*$.

\vspace{1.5mm}
\Statex \textbf{\textit{Phase 2: Out-of-Sample Prediction (Inference)}}
\vspace{1.5mm}
\State \textbf{Initialize:} $\hat{f}_{\text{new}}(\boldsymbol{s}) \leftarrow f_0(\boldsymbol{s}), \quad \forall \boldsymbol{s} \in \mathcal{S}$
\For{$m = 1$ to $p$}
    \State Calculate target shift: $\Delta\tau_{\text{new}}^m(\boldsymbol{s}) = \tau_{\text{new}}^m(\boldsymbol{s}) - \tau_0^m(\boldsymbol{s})$
    \State Compute effective flow amount: $t_m(s) = g_m(\Delta\tau_{\text{new}}^m(\boldsymbol{s}))$
    \State Predict latent coordinates via ODE solver: $\hat{f}_{\text{new}}(\boldsymbol{s}) \leftarrow \hat{f}_{\text{new}}(\boldsymbol{s}) + \exp \left( t_m(\boldsymbol{s}) V_m \right)$
\EndFor
\State \textbf{Reconstruct Covariance:}
\For{$\boldsymbol{s}_i, \boldsymbol{s}_j \in \mathcal{S}$}
    \State $\hat{\Sigma}_{\text{new}}(\boldsymbol{s}_i, \boldsymbol{s}_j) = C_{\boldsymbol{\theta}}(||f_{\text{new}}(\boldsymbol{s}_i) - f_{\text{new}}(\boldsymbol{s}_j)||)$
\EndFor
\\State \Return $\hat{\Sigma}_{\text{new}}$
\end{algorithmic}
\end{spacing}
\end{algorithm}

The implementation is divided into two phases. 
We discuss their implementation details separately: 
\begin{enumerate}
    \item In phase 1, the objective is to estimate the channel-wise base velocity fields $V_1, \cdots, V_p$ and link functions $g_1, \cdots, g_p$. 
    Because $V_m$ are continuous vector fields, and $g_m$ are continuous, strictly monotonic functions, both of them can be parameterized via flexible function approximators, e.g., basis expansions for limited sample sizes or NNs with sufficient samples. 
    While ODE gives the definition of exponential mappings on a Lie group, multiple efficient approximations have been established \citep{higham2008functions, al2010new}, and therefore, the computational burden of this algorithm is acceptable. 
    Notably, we perform the estimation via forward integration in the deformed coordinate space, rather than via regression in the Lie algebra. 
    This is because the Lie logarithm mapping of a diffeomorphism is numerically ill-posed and highly sensitive to spatial fluctuations \citep{beg2005computing, hernandez2018newton}, which do exist in the deformations $f_k$ empirically estimated from noisy observations. 
    Conversely, the exponential mapping via forward ODE integration inherently acts as a smoothing operator in practice \citep{polzin2018large}, yielding a more stable optimization landscape.  
    The targets $V_m$ and $g_m$ are jointly updated by minimizing the spatial discrepancy (e.g., mean squared error) between the model-predicted deforemd coordinates $\hat{f}_k(\boldsymbol{s})$ and the empirically deformed coordinates $f_k(\boldsymbol{s})$. 
    \item In phase 2, the model is deployed to predict the nonstationary covariance for a novel covariate condition, based on the velocity fields and link functions estimated. 
    The algorithm computes the predicted deformed coordinates in the latent space $\mathcal{W}$ driven by the learned dynamics.
    Thus, the covariance matrix $\Sigma_{\text{new}}$ can be reconstructed with $C_{\boldsymbol{\theta}}$ and the predicted deformed coordinates, as the base process in the latent space $\mathcal{W}$ inherently maintains isotropic. 
\end{enumerate}

\section{Experimental Studies} \label{sec_experiment}

In this section, we present: (i) A designed simulation study verifies the out-of-sample prediction capability of the covariate-driven model. (ii) We apply the method to the AM surface deviation dataset to rigorously benchmark our predictive performance against established spatial modeling techniques. (iii) The terrain temperature application demonstrates the method's broad generalizability across diverse physical domains, providing a visual illustration of how local terrain explicitly governs the velocity fields as well as the resulting spatial deformations. 

\subsection{Simulation Study} 

The primary challenge in validating nonstationary spatial modeling on physical data is the absence of a ground-truth covariance structure, as the experimental datasets often provide a single, noisy realization. 
Therefore, we design this simulation study where the true base process, the deformation mechanism, and the resulting true nonstationary covariances are analytically defined. 

We establish a two-dimensional continuous spatial domain $\mathcal{S} = \mathcal{W} = [-1, 1]^2$. 
The latent isotropic base process is governed by a Mat\'ern kernel $C_{\boldsymbol{\theta}}$. 
Two independent covariate channels, $\tau^1$ and $\tau^2$ are defined, which manipulate the spatial deformation via two distinct base velocity fields $V_1$ and $V_2$: 
$$
    V_1 = [\sin (\pi s_x), 0]^\top,
    \quad 
    V_2 = [0, \exp(-5 s_y^2)]^\top,
$$
where $\boldsymbol{s} = (s_x, s_y)^\top \in \mathcal{S}$. 

\begin{figure}[!htp]
    \centering{
        \includegraphics[width=0.6\textwidth]{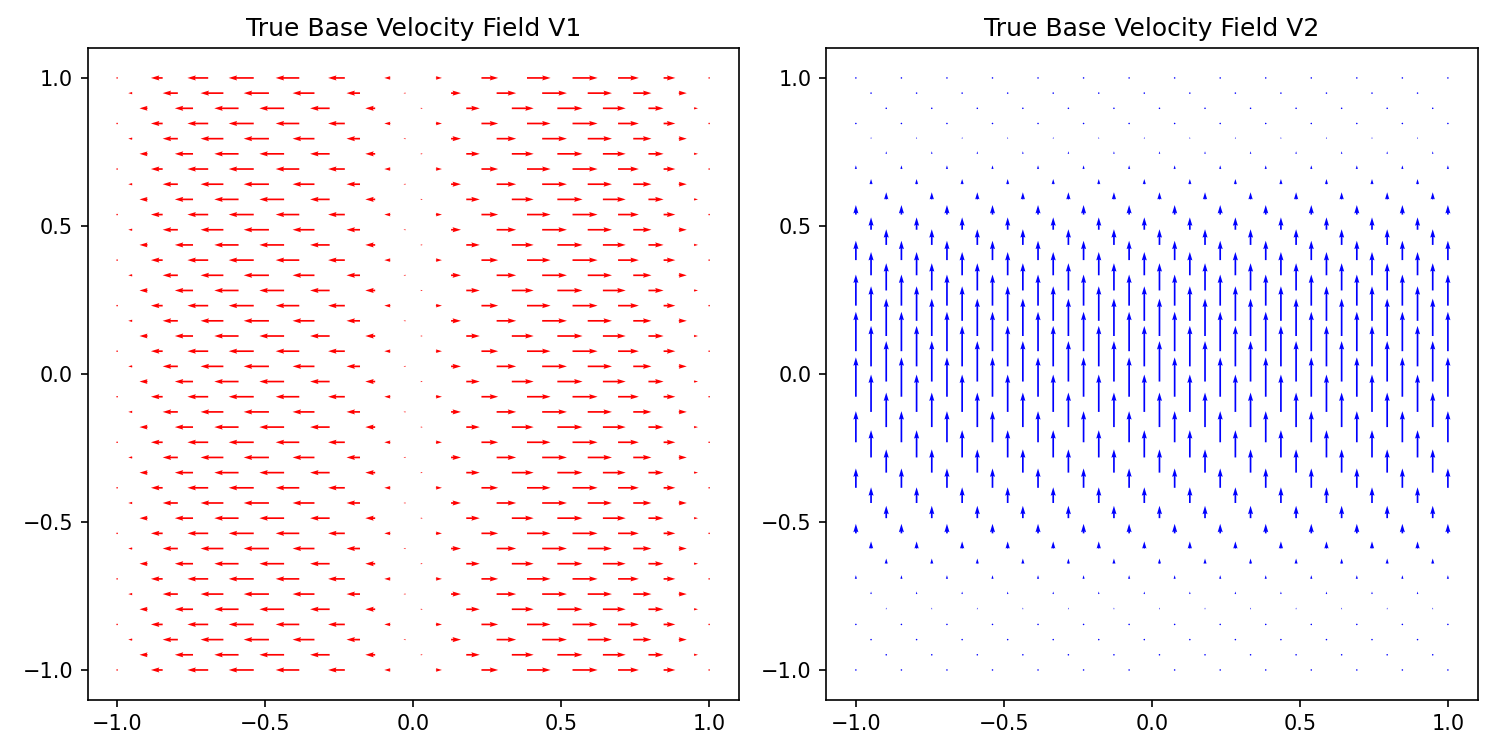}
    }
    \caption{\label{fig_simu_true_v} True base velocity fields $V_1$ and $V_2$ associated with covariate channel $\tau^1$ and $\tau^2$, respectively.}
\end{figure}

The two velocity fields are also plotted in Fig.~\ref{fig_simu_true_v}. This satisfies the Design of Experiments path independence condition (Assumption 2) required for stable estimation. 
However, despite this commutativity, the underlying functions are highly nonlinear.
Fig.~\ref{fig_simu_deformed_realization} shows that, by altering the values of $\tau^1$ and $\tau^2$, the combination of these fields can produce varying nonstationary covariance structures. 

\begin{figure}[!htp]
    \centering{
        \includegraphics[width=0.95\textwidth]{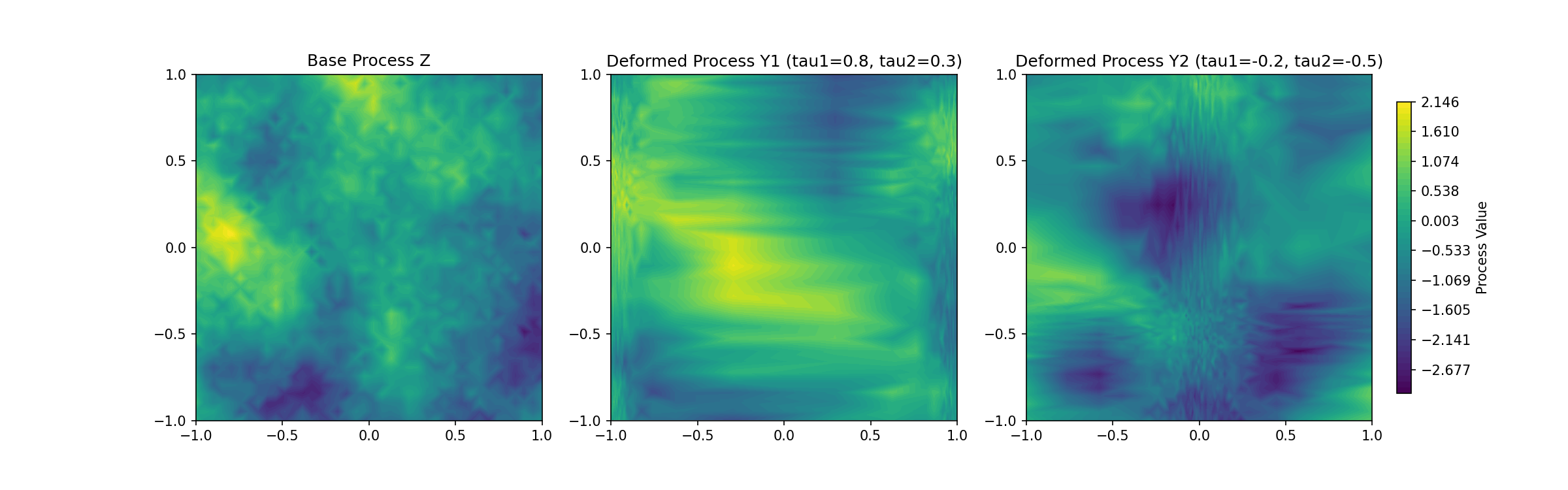}
    }
    \caption{\label{fig_simu_deformed_realization} With the base velocity fields defined as Fig.~\ref{fig_simu_true_v}, the isotropic base process $Z$ (left) can be deformed into distinct nonstationary GPs (middle and right) with different covariate conditions. The three subplots are the same realization but with different deformations.}
\end{figure}

We simulate an extreme data-scarse scenario, with only four observations from distinct covariate conditions available, as shown in Tab.~\ref{tab_simu_covariate}. 
Also, the new covariate condition to be predicted is out of the range of the observed samples, asking for the model to have a capability of extrapolation. 

\begin{table}[!htp]
    \centering
    \caption{\label{tab_simu_covariate}Simulated covariate conditions of training samples $k = 1, \cdots, 4$, and new conditions to be predicted.}
    \begin{tabular}{c|cccc|c}
        \toprule
        $k$ & 1 & 2 & 3 & 4 & new \\
        \midrule
        $\tau^1$ & 0.0 & 0.5 & 0.8 & 0.4 & 0.3 \\
        $\tau^2$ & 0.0 & 0.1 & 0.1 & 0.7 & -0.5  \\
        \bottomrule
    \end{tabular}
\end{table}

Following the pipeline in Algorithm~\ref{alg_pipeline}, we employ the tensor product of b-spline bases as the function approximator, and assume $g_m(\Delta \tau^m) = \Delta \tau^m$, as the observations are truly sparse. 
After phase 1, the estimated velocity fields are plotted in black in Fig.~\ref{fig_simu_estimated_v}, while the true velocity fields are plotted in color. 
From the figures, we can see that the two velocity fields are successfully recovered, with minor errors induced by estimation or added noise in the observation. 
Further, we use them to make the out-of-sample prediction of the covariance structure driven by $\boldsymbol{\tau}_{\text{new}}$, and the comparison between the pixel-wise covariance matrices is shown in Fig.~\ref{fig_simu_estimated_cov}. 
The nonstationary covariance structure is fully predicted by the model in this data-scarce scenario, indicating the strong capability of capturing the driving mechanism of nonstationarity and making extrapolative predictions at unseen covariate conditions. 

\begin{figure}[!htp]
    \centering{
        \includegraphics[width=0.6\textwidth]{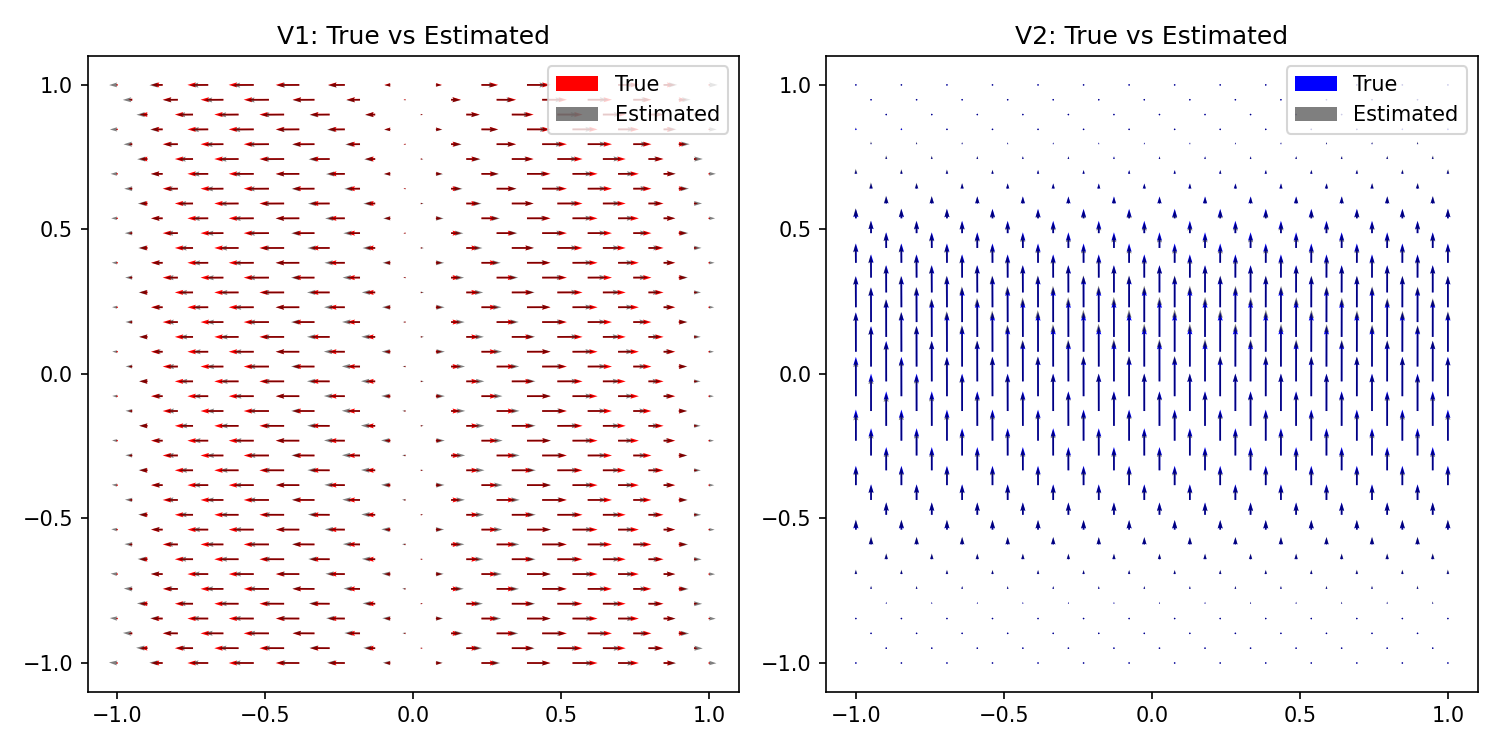}
    }
    \caption{\label{fig_simu_estimated_v} The estimated velocity fields (grey) are close to the underlying truth (red/blue).}
\end{figure}

\begin{figure}[!htp]
    \centering{
        \includegraphics[width=0.6\textwidth]{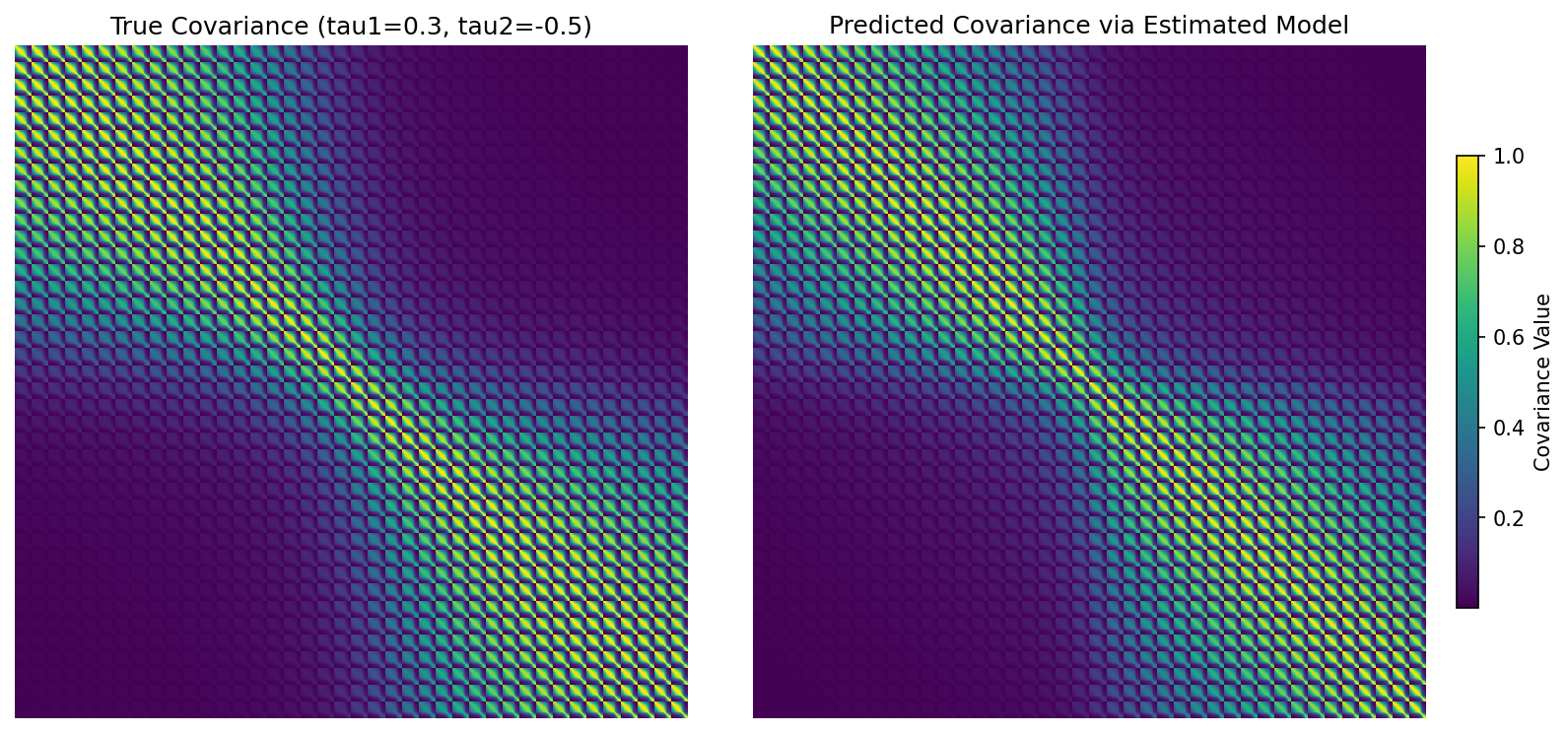}
    }
    \caption{\label{fig_simu_estimated_cov} The comparison between pixel-wise covariance matrices: the underlying true at $\boldsymbol{\tau}_{\text{new}}$ (left) and the out-of-sample prediction (right).}
\end{figure}

\subsection{Case Study: Additive Manufacturing Surface Deviations} 

To evaluate the proposed method in a real-world physical context, we apply it to the surface deviation distribution prediction of FDM-fabricated AM products. 
As originally illustrated in Fig.~\ref{fig_example}, the surface deviation pattern of printed parts exhibits strong nonstationarity jointly impacted by local surface geometry and process-related covariates \citep{gu2026surface}. 

We focus on the two dome-shaped parts with different sizes, $\mathcal{D}_1$ and $\mathcal{D}_2$, $\mathcal{D}_1$ for training and $\mathcal{D}_2$ for testing, as shown in Fig.~\ref{fig_am_pipeline}(a) and (d). 
While the entire printed product constitutes a surface manifold, learning a single, global set of continuous velocity fields to map this entire topology is mathematically prohibitive. 
Therefore, we pre-process the dome surfaces by partitioning them into localized spherical patches each, surface deviation observations on the 50 patches from $\mathcal{D}_1$ compose the training dataset $\mathcal{Y}_{\text{train}} = \{ Y_k \}_{k = 1}^{50}$, while the 50 patches from $\mathcal{D}_2$ the testing dataset $\mathcal{Y}_{\text{test}}$. 
The azimuthal angle, the polar angle and the in-layer radius of the patch are selected as the three channels of the covariate vector $\boldsymbol{\tau}$, as defined in Fig.~\ref{fig_am_pipeline}(b). 
They are selected based on the domain knowledge on AM surface quality \citep{gu2025identification, gu2026surface}. 

\begin{figure}[!htp]
    \centering{
        \includegraphics[width=0.9\textwidth]{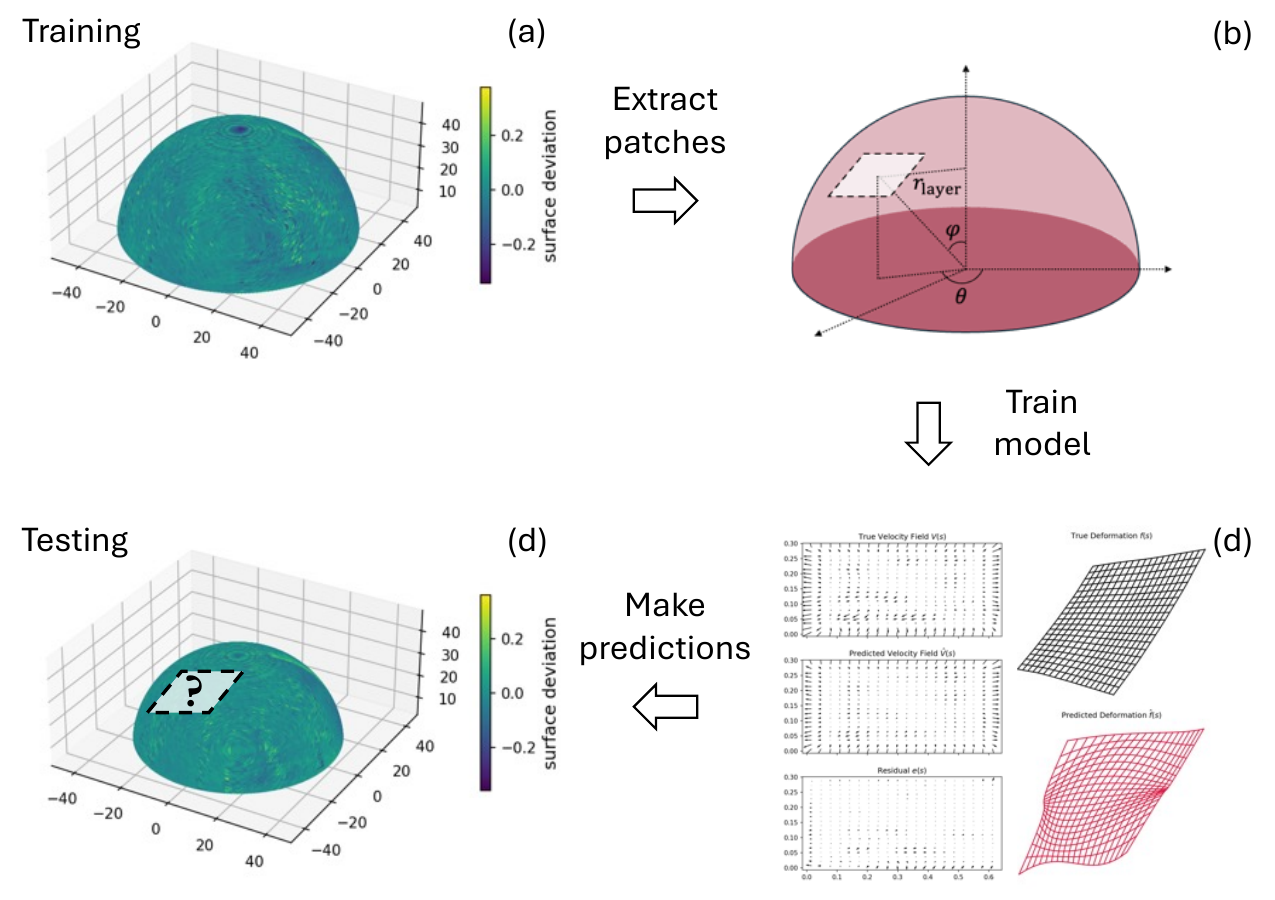}
    }
    \caption{\label{fig_am_pipeline} (a) Dome $\mathcal{D}_1$ for training. (b) Illustration of patch covariates: $\tau^1 = \theta$ is the azimuthal angle; $\tau^2 = \varphi$ is the polar angle; $\tau^3 = r_{\text{layer}}$ is the in-layer radius. (c) The model is trained following the pipeline in Algorithm~\ref{alg_pipeline}. (d) Dome $\mathcal{D}_2$ for testing. }
\end{figure}

\begin{figure}[!htp]
    \centering{
        \includegraphics[width=0.8\textwidth]{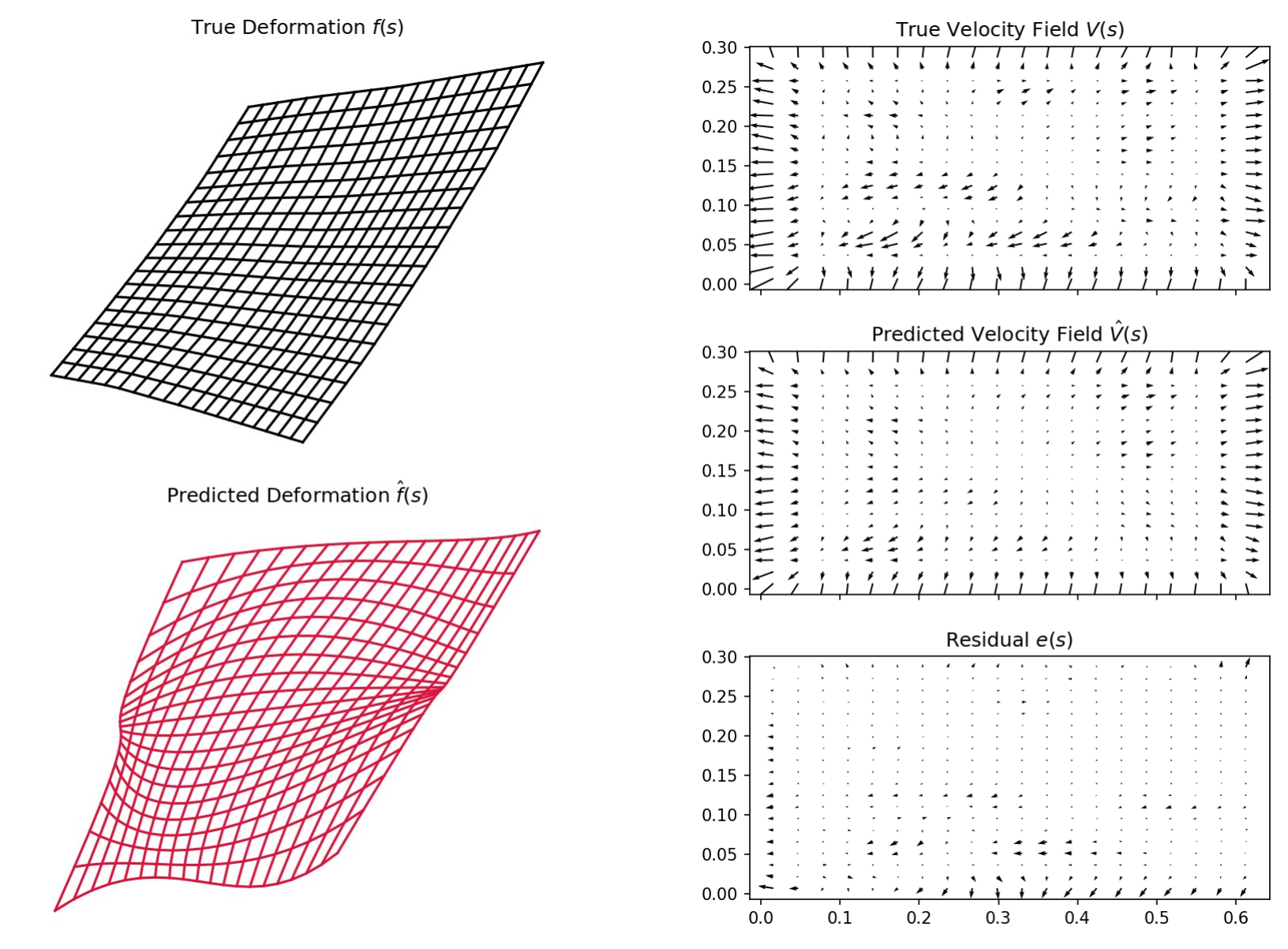}
    }
    \caption{\label{fig_am_prediction} (a) Dome $\mathcal{D}_1$ for training. (b) Illustration of patch covariates: $\tau^1 = \theta$ is the azimuthal angle; $\tau^2 = \varphi$ is the polar angle; $\tau^3 = r_{\text{layer}}$ is the in-layer radius. (c) The model is trained following the pipeline in Algorithm~\ref{alg_pipeline}. (d) Dome $\mathcal{D}_2$ for testing. }
\end{figure}

We follow the pipeline in Algorithm~\ref{alg_pipeline}, estimating the phase 1 model with the patches in $\mathcal{Y}_{\text{train}}$ and making deformation/covariance predictions for $\mathcal{Y}_{\text{test}}$. 
The left of Fig.~\ref{fig_am_prediction} presents the comparison between the true deformation $f_{\text{test},k}$ of a patch in $\mathcal{Y}_{\text{test}}$ and its predicted deformation $\hat{f}_{\text{test},k}$. 
From the comparison, the prediction captures the main deformation trend required for restoring isotropy in this patch, while some minor inconsistencies are potentially due to local numerical instability. 
We further plot the true and predicted velocity fields, $V_{\text{test},k}$ and $\hat{V}_{\text{test},k}$, repectively, driving such deformations, in the right of Fig.~\ref{fig_am_prediction}. 
The residual $e(\boldsymbol{s}) = V_{\text{test},k}(\boldsymbol{s}) - \hat{V}_{\text{test},k}(\boldsymbol{s})$ is also plotted, and the error is very small compared to the true velocity field.

As the underlying true covariances of patches in $\mathcal{Y}_{\text{test}}$ are unknown, we employ the log-likelihood of surface deviation profiles with respect to the covariances predicted as the quantitative evaluation metric: $\log \ell \left( \hat{\Sigma}_{\text{test},k} \vert Y_{\text{test},k} \right) = \log \mathbb{P} \left( Y_{\text{test},k} \vert \hat{\Sigma}_{\text{test},k} \right)$. 

To rigorously evaluate the capability of our proposed method, we compare the likelihood from our predictive model to three other benchmarks: 
\begin{enumerate}
    \item Na\"ive stationary GP model: We first employ the conventional stationary GP model with Mat\'ern kernel to compare whether our method can capture the heterogeneity due to covariate differences. 
    \item Automatic relevance determination (ARD) GP model: ARD is a highly popular and established model in spatial statistics and machine learning in handling multiple covariates \citep{bishop2006pattern, seeger2004gaussian}. 
    \item Unstructured NN deformation model: One may question whether the functional relationship $f_k(\cdot) = f(\cdot, \boldsymbol{\tau}_k)$ can be directly captured by an NN, without any topological or algebraic constraints enforced. Therefore, we compare our method established in Sec.~\ref{sec_method} with a multilayer perceptron (MLP) model. 
\end{enumerate}

The comparison between our method and the benchmarks is shown in Fig.~\ref{fig_am_result}.
Each column of the plot reflects the distribution of pair-wise likelihood differences, $\log \ell_{\text{benchmark},k} - \log \ell_{\text{method},k}$. 
Therefore, the negative values indicate the superiority of our proposed method to the specific benchmark, and vice versa. 
The detailed statistics of likelihood differences are also summarized in Tab.~\ref{tab_am_result}, in which the $p$-value is for the pair-wise $t$-test, examining whether there is a significant mean value difference in the likelihood. 

\begin{figure}[!htp]
    \centering{
        \includegraphics[width=0.6\textwidth]{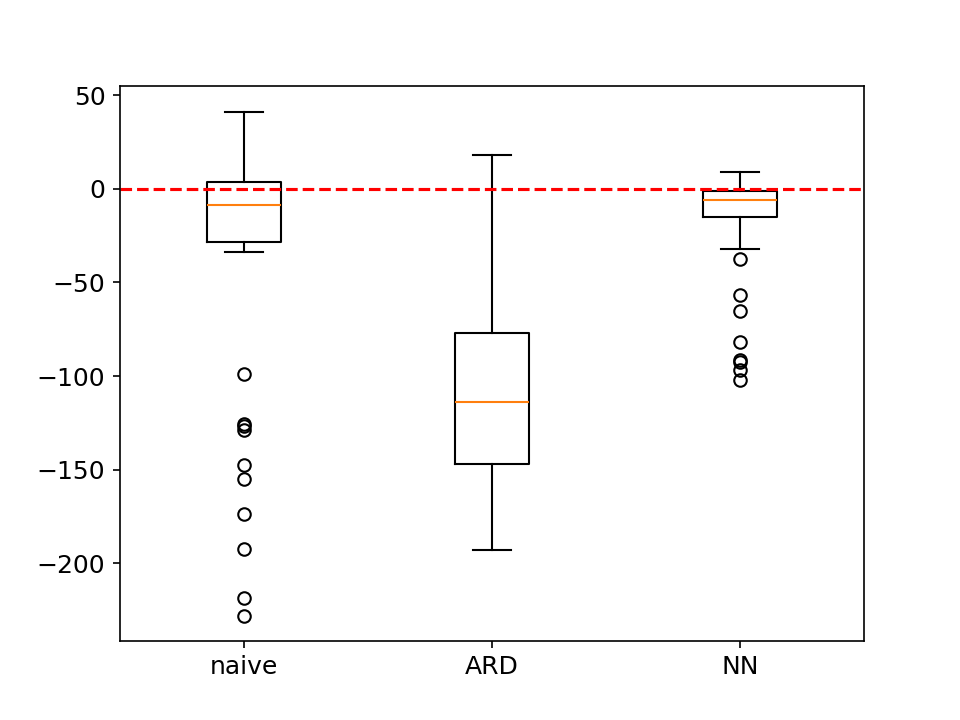}
    }
    \caption{\label{fig_am_result} Boxplot of pair-wise likelihood differences $\log \ell_{\text{benchmark},k} - \log \ell_{\text{method},k}$: Our proposed method outperforms all three benchmark methods.}
\end{figure}

\begin{table}[!htp]
    \centering
    \caption{\label{tab_am_result}Statistics of likelihood differences between benchmark methods and our proposed method.}
    \begin{tabular}{c|cc|c}
        \toprule
        Benchmark method & $\mu(\log \ell_{\text{b}} - \log \ell_{\text{m}})$ & $\sigma(\log \ell_{\text{b}} - \log \ell_{\text{m}})$ & $p$-value \\
        \midrule
        Na\"ive & 33.70 & 68.25 & 0.00051 \\
        ARD & 111.10 & 45.89 & 0 \\
        NN & 17.57 & 28.95 & 0.00004 \\
        \bottomrule
    \end{tabular}
\end{table}

From both the boxplot and the table, the out-of-sample predictive performance of our method consistently outperforms all three benchmarks. 
The na\"ive stationary GP model exhibits a moderate level of predictive power but is troubled by severe, extreme negative outliers. 
This indicates that while a purely spatial stationary model might perform adequately on test patches that reflect the global average state, it fails catastrophically when dealing with some strong nonstationarity driven by covariates. 
The ARD GP demonstrates uniformly poor predictive capability. 
This confirms that the nonstationarity inherent in AM surface deviation is too complex to be captured by simply scaling individual dimensions linearly via ARD. 
Notably, while the unstructured NN deformation model achieves a performance closest to our proposed method, it suffers from a pronounced tail of negative outliers. 
Although the deep NN acts as a highly flexible universal approximator, its lack of geometric regularization makes it prone to spatial folding and overfitting when trained on limited sample sizes.
This provides strong empirical evidence for the necessity of our formulation. 

\subsection{Case Study: Terrain Temperatures} 

\begin{figure}[!htp]
    \centering{
        \includegraphics[width=0.9\textwidth]{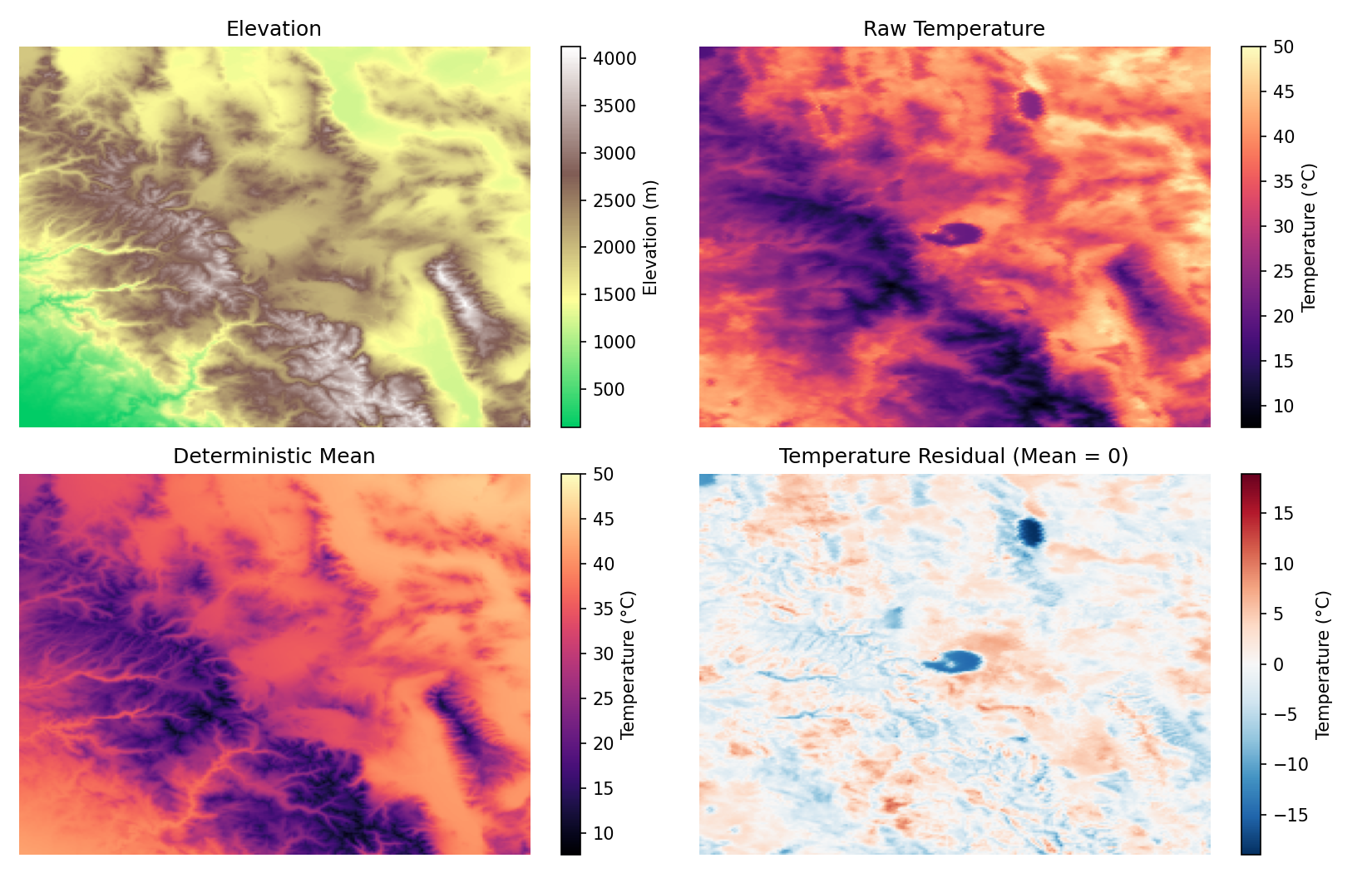}
    }
    \caption{\label{fig_temperature}The high-resolution terrain and temperature observation over a region of interest \citep{wan2015mod11a2, nasa2013shuttle}. Top left: elevation map; top right: raw summer temperature measurement; bottom left: deteministc mean temperature given by climatology models; bottom right: mean-zero temperature residual.}
\end{figure}

To demonstrate the broad generalizability and physical interpretability of the proposed covariate-driven method beyond AM, we apply it to a terrain temperature dataset. 
The top two subplots of Fig.~\ref{fig_temperature} illustrates the elevation and the summer temperature in a mountainous region. 
Specifically, this region of interest is located between 120 and 118 degrees west longitude and between 37 and 39 degrees north latitude. 
After extracting the deterministic mean (bottom left) given by advanced climatological models \citep{hijmans2005very}, the residual temperature field (bottom right) still exhibits complex spatial nonstationarity. 
While traditional models treat these residual variations as arbitrary spatial noise, visual inspection strongly suggests they are modulated by the underlying terrain topography. 

The objective of this exploratory experiment is qualitative: to evaluate the model's capacity for autonomous physical discovery. 
We utilize the localized high-resolution SRTM elevation profile as the sole covariate channel. 
By training the model to learn deformations based exclusively on local elevation, we investigate whether the resulting velocity field successfully captures a physically meaningful relationship. 
Specifically, we aim to observe if the predicted dynamics naturally align with the topological reality of the landscape. 
Similar to the previous AM experiment, we partition the region into smaller patches and fit the model within the patches, following the pipeline in Algorithm~\ref{alg_pipeline}. 

\begin{figure}[!htp]
    \centering{
        \includegraphics[width=0.95\textwidth]{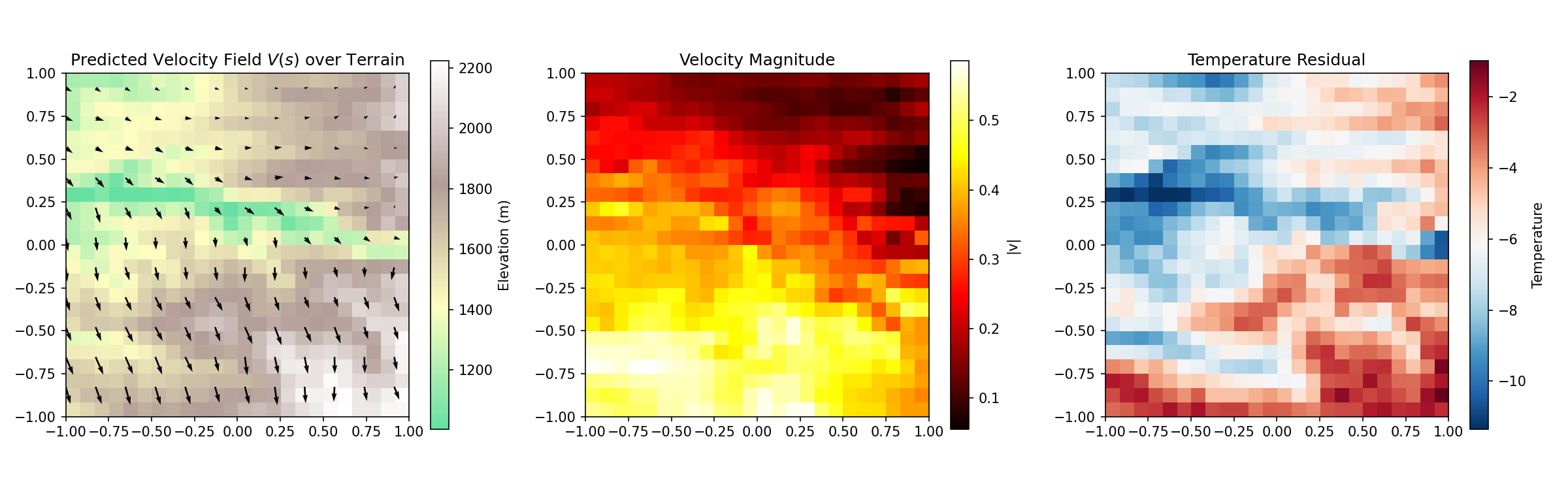}
        \includegraphics[width=0.95\textwidth]{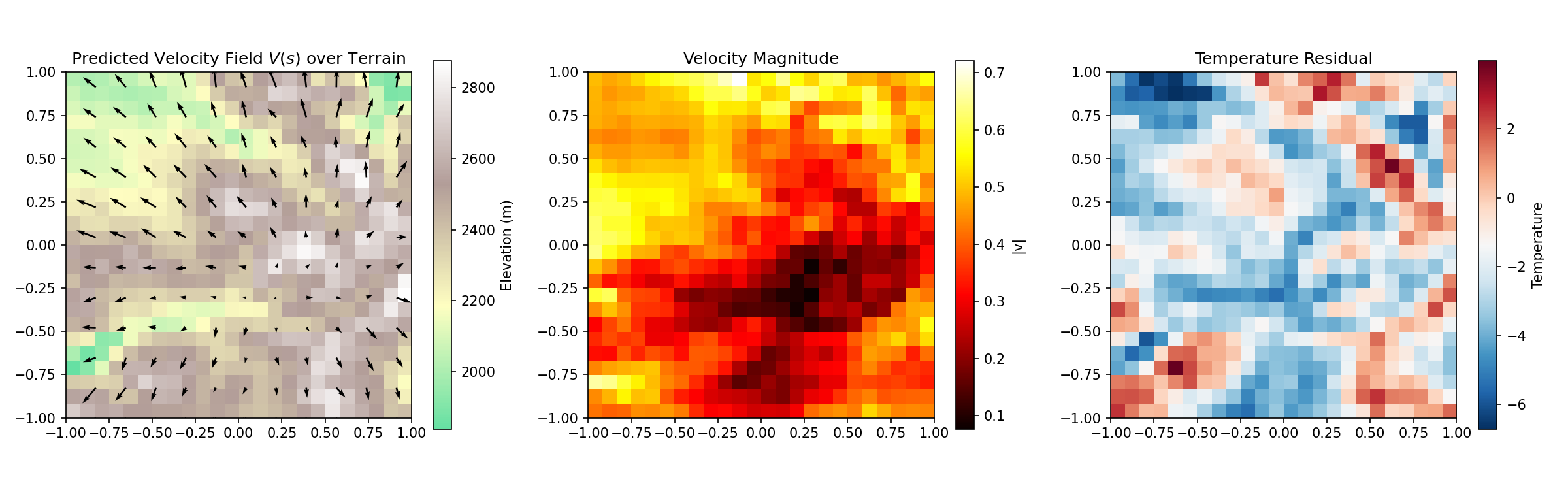}
    }
    \caption{\label{fig_temp_pred} Predicted velocity fields (left), velocity magnitude, and true temperatures over two representative terrain patches.}
\end{figure}

Based on the model estimated, we predict velocity fields for two representative patches with complex local terrain, as plotted in Fig.~\ref{fig_temp_pred}. 
We need to notice here that whether the direction of the velocity fields is pointing uphill or downhill is not so different (the specific direction can be related to the anchoring of patches), as both mean stretching the local domain apart. 
Rather, the magnitude of velocity should be emphasized here. 
From the predicted velocity fields (left) and the velocity magnitude plots (middle), we can clearly see that in areas with dramatic topographic changes, the velocity field magnitude is also greater. 
Physically speaking, this means that areas with larger elevation spans are often accompanied by significant temperature variations, and therefore, deformation driven by the velocity field tends to stretch and separate uncorrelated regions. 
Conversely, in regions with more isotropic topography, such as plains or mountain peaks, the velocity magnitude is very small, corresponding to a stronger local correlation. 
Thus, our proposed model greatly aligns the temperature dynamics with the terrain topology well. 

\section{Conclusion} \label{sec_conclusion}

This paper introduced a novel method, extending the conventional deformation method to enable out-of-sample prediction for nonstationary GPs. 
By characterizing spatial deformations through the lens of Lie algebra and considering the commutativity between multiple covariates, we established a concise, explicit functional form of multi-covariate-driven spatial deformations. 
And based on the theoretical base, we presented the algorithm to efficiently train the model of covariate-driven deformations with limited sample sizes, and to make predictions of nonstationary GPs under novel, unseen covariate conditions. 

The various empirical studies demonstrate the capability of our method. 
In the simulated environment, the model successfully disentangled and recovered the true underlying physical dynamics from noisy observations. 
In the additive manufacturing study, our method significantly outperformed multiple benchmarks, proving that the geometric regularization of velocity fields is critical for accurately extrapolating topological warping under novel manufacturing conditions. 
Furthermore, the terrain temperature case study highlighted the model's capacity for autonomous physical discovery, naturally learning how topological barriers disrupt spatial correlation.

Several promising avenues remain for future research. 
First, while the current framework assumes independent covariate channels, many complex engineering systems exhibit strong interaction effects. 
Future work could relax the strict commutativity assumption by introducing structured penalty terms for the Lie brackets, allowing for controlled geometric interactions without entirely destabilizing the optimization landscape. 
Second, the current estimation of the velocity fields is deterministic. 
Developing a fully Bayesian formulation for the Neural ODE parameters would allow for rigorous uncertainty quantification not only of the predicted GP but of the spatial deformation dynamics themselves. 

\section*{Acknowledgement}

This work is partially supported by the National Science Foundation with Grant\# CMMI–2328455. 

Land surface temperature data are acquired from the MODIS Terra MOD11A2 Version 6.1 product \citep{wan2015mod11a2}, temporally averaged over the period of June 1, 2023, to September 30, 2023, to establish a stable summer baseline. Elevation data are acquired from the SRTM 1 Arc-Second Global Version 3 dataset \citep{nasa2013shuttle}. Both are accessed and spatially aligned to a 1 km resolution grid via Google Earth Engine \citep{gorelick2017google}.

\bibliography{bibliography.bib}

@article{kennedy2001bayesian,
  title={Bayesian calibration of computer models},
  author={Kennedy, Marc C and O'Hagan, Anthony},
  journal={Journal of the Royal Statistical Society: Series B (Statistical Methodology)},
  volume={63},
  number={3},
  pages={425--464},
  year={2001},
  publisher={Wiley Online Library}
}

@book{stein1999interpolation,
  title={Interpolation of spatial data: some theory for kriging},
  author={Stein, Michael L},
  year={1999},
  publisher={Springer Science \& Business Media}
}

@article{seeger2004gaussian,
  title={Gaussian processes for machine learning},
  author={Seeger, Matthias},
  journal={International journal of neural systems},
  volume={14},
  number={02},
  pages={69--106},
  year={2004},
  publisher={World Scientific}
}

@article{matos2020improving,
  title={Improving additive manufacturing performance by build orientation optimization},
  author={Matos, Marina A and Rocha, Ana Maria AC and Pereira, Ana I},
  journal={The International Journal of Advanced Manufacturing Technology},
  volume={107},
  number={5},
  pages={1993--2005},
  year={2020},
  publisher={Springer}
}

@article{gu2026surface,
  title={Surface Quality Characterization, Learning and Prediction for 3D Printed Two-dimensional Free-form Products Through Dimensional Reduction},
  author={Gu, Minghao and Ruiz, Cesar and Huang, Qiang},
  journal={IEEE Transactions on Automation Science and Engineering},
  year={2026},
  publisher={IEEE}
}

@article{wan2014new,
  title={New refinements and validation of the collection-6 MODIS land-surface temperature/emissivity product},
  author={Wan, Zhengming},
  journal={Remote sensing of Environment},
  volume={140},
  pages={36--45},
  year={2014},
  publisher={Elsevier}
}

@article{farr2007shuttle,
  title={The shuttle radar topography mission},
  author={Farr, Tom G and Rosen, Paul A and Caro, Edward and Crippen, Robert and Duren, Riley and Hensley, Scott and Kobrick, Michael and Paller, Mimi and Rodriguez, Ernesto and Roth, Ladislav and others},
  journal={Reviews of geophysics},
  volume={45},
  number={2},
  year={2007},
  publisher={Wiley Online Library}
}

@article{hijmans2005very,
  title={Very high resolution interpolated climate surfaces for global land areas},
  author={Hijmans, Robert J and Cameron, Susan E and Parra, Juan L and Jones, Peter G and Jarvis, Andy},
  journal={International Journal of Climatology: A Journal of the Royal Meteorological Society},
  volume={25},
  number={15},
  pages={1965--1978},
  year={2005},
  publisher={Wiley Online Library}
}

@article{fouedjio2017second,
  title={Second-order non-stationary modeling approaches for univariate geostatistical data},
  author={Fouedjio, Francky},
  journal={Stochastic Environmental Research and Risk Assessment},
  volume={31},
  number={8},
  pages={1887--1906},
  year={2017},
  publisher={Springer}
}

@article{sauer2023non,
  title={Non-stationary Gaussian process surrogates},
  author={Sauer, Annie and Cooper, Andrew and Gramacy, Robert B},
  journal={arXiv preprint arXiv:2305.19242},
  year={2023}
}

@article{risser2016nonstationary,
  title={Nonstationary spatial modeling, with emphasis on process convolution and covariate-driven approaches},
  author={Risser, Mark D},
  journal={arXiv preprint arXiv:1610.02447},
  year={2016}
}

@article{sampson1992nonparametric,
  title={Nonparametric estimation of nonstationary spatial covariance structure},
  author={Sampson, Paul D and Guttorp, Peter},
  journal={Journal of the American Statistical Association},
  volume={87},
  number={417},
  pages={108--119},
  year={1992},
  publisher={Taylor \& Francis}
}

@misc{higdon1999non,
  title={Non-stationary spatial modeling. Bayesian Statistics 6, eds. J. Bernardo, J. Berger, A. Dawid, and A. Smith},
  author={Higdon, D and Swall, J and Kern, J},
  year={1999},
  publisher={Oxford, UK: Oxford University Press}
}

@article{paciorek2003nonstationary,
  title={Nonstationary covariance functions for Gaussian process regression},
  author={Paciorek, Christopher and Schervish, Mark},
  journal={Advances in neural information processing systems},
  volume={16},
  year={2003}
}

@inproceedings{heinonen2016non,
  title={Non-stationary gaussian process regression with hamiltonian monte carlo},
  author={Heinonen, Markus and Mannerstr{\"o}m, Henrik and Rousu, Juho and Kaski, Samuel and L{\"a}hdesm{\"a}ki, Harri},
  booktitle={Artificial Intelligence and Statistics},
  pages={732--740},
  year={2016},
  organization={PMLR}
}

@article{binois2018practical,
  title={Practical heteroscedastic Gaussian process modeling for large simulation experiments},
  author={Binois, Mickael and Gramacy, Robert B and Ludkovski, Mike},
  journal={Journal of Computational and Graphical Statistics},
  volume={27},
  number={4},
  pages={808--821},
  year={2018},
  publisher={Taylor \& Francis}
}

@article{lindgren2011explicit,
  title={An explicit link between Gaussian fields and Gaussian Markov random fields: the stochastic partial differential equation approach},
  author={Lindgren, Finn and Rue, H{\aa}vard and Lindstr{\"o}m, Johan},
  journal={Journal of the Royal Statistical Society Series B: Statistical Methodology},
  volume={73},
  number={4},
  pages={423--498},
  year={2011},
  publisher={Oxford University Press}
}

@article{nychka2002multiresolution,
  title={Multiresolution models for nonstationary spatial covariance functions},
  author={Nychka, Douglas and Wikle, Christopher and Royle, J Andrew},
  journal={Statistical Modelling},
  volume={2},
  number={4},
  pages={315--331},
  year={2002},
  publisher={Sage Publications Sage CA: Thousand Oaks, CA}
}

@article{katzfuss2017multi,
  title={A multi-resolution approximation for massive spatial datasets},
  author={Katzfuss, Matthias},
  journal={Journal of the American Statistical Association},
  volume={112},
  number={517},
  pages={201--214},
  year={2017},
  publisher={Taylor \& Francis}
}

@article{reich2011class,
  title={A class of covariate-dependent spatiotemporal covariance functions},
  author={Reich, Brian J and Eidsvik, Jo and Guindani, Michele and Nail, Amy J and Schmidt, Alexandra M},
  journal={The annals of applied statistics},
  volume={5},
  number={4},
  pages={2265},
  year={2011}
}

@article{neto2014accounting,
  title={Accounting for spatially varying directional effects in spatial covariance structures},
  author={Neto, Joaquim Henriques Vianna and Schmidt, Alexandra M and Guttorp, Peter},
  journal={Journal of the Royal Statistical Society Series C: Applied Statistics},
  volume={63},
  number={1},
  pages={103--122},
  year={2014},
  publisher={Oxford University Press}
}

@article{ingebrigtsen2014spatial,
  title={Spatial models with explanatory variables in the dependence structure},
  author={Ingebrigtsen, Rikke and Lindgren, Finn and Steinsland, Ingelin},
  journal={Spatial Statistics},
  volume={8},
  pages={20--38},
  year={2014},
  publisher={Elsevier}
}

@article{paciorek2006spatial,
  title={Spatial modelling using a new class of nonstationary covariance functions},
  author={Paciorek, Christopher J and Schervish, Mark J},
  journal={Environmetrics: The official journal of the International Environmetrics Society},
  volume={17},
  number={5},
  pages={483--506},
  year={2006},
  publisher={Wiley Online Library}
}

@article{fuglstad2015does,
  title={Does non-stationary spatial data always require non-stationary random fields?},
  author={Fuglstad, Geir-Arne and Simpson, Daniel and Lindgren, Finn and Rue, H{\aa}vard},
  journal={Spatial Statistics},
  volume={14},
  pages={505--531},
  year={2015},
  publisher={Elsevier}
}

@article{bornn2012modeling,
  title={Modeling nonstationary processes through dimension expansion},
  author={Bornn, Luke and Shaddick, Gavin and Zidek, James V},
  journal={Journal of the American statistical association},
  volume={107},
  number={497},
  pages={281--289},
  year={2012},
  publisher={Taylor \& Francis}
}

@article{rasmussen2001infinite,
  title={Infinite mixtures of Gaussian process experts},
  author={Rasmussen, Carl and Ghahramani, Zoubin},
  journal={Advances in neural information processing systems},
  volume={14},
  year={2001}
}

@article{kim2005analyzing,
  title={Analyzing nonstationary spatial data using piecewise Gaussian processes},
  author={Kim, Hyoung-Moon and Mallick, Bani K and Holmes, Chris C},
  journal={Journal of the American Statistical Association},
  volume={100},
  number={470},
  pages={653--668},
  year={2005},
  publisher={Taylor \& Francis}
}

@article{gramacy2008bayesian,
  title={Bayesian treed Gaussian process models with an application to computer modeling},
  author={Gramacy, Robert B and Lee, Herbert K H},
  journal={Journal of the American Statistical Association},
  volume={103},
  number={483},
  pages={1119--1130},
  year={2008},
  publisher={Taylor \& Francis}
}

@article{gramacy2015local,
  title={Local Gaussian process approximation for large computer experiments},
  author={Gramacy, Robert B and Apley, Daniel W},
  journal={Journal of Computational and Graphical Statistics},
  volume={24},
  number={2},
  pages={561--578},
  year={2015},
  publisher={Taylor \& Francis}
}

@article{tresp2000mixtures,
  title={Mixtures of Gaussian processes},
  author={Tresp, Volker},
  journal={Advances in neural information processing systems},
  volume={13},
  year={2000}
}

@article{konomi2014bayesian,
  title={Bayesian treed multivariate gaussian process with adaptive design: Application to a carbon capture unit},
  author={Konomi, Bledar and Karagiannis, Georgios and Sarkar, Avik and Sun, Xin and Lin, Guang},
  journal={Technometrics},
  volume={56},
  number={2},
  pages={145--158},
  year={2014},
  publisher={Taylor \& Francis}
}

@misc{perrin1998modeling,
  title={Modeling of non-stationary spatial covariance structure by parametric radial basis deformations, volume 11 of Quantitative Geology and Geostatistics},
  author={Perrin, O and Monestiez, P},
  year={1998},
  publisher={Springer Netherlands}
}

@article{iovleff2004estimating,
  title={Estimating a nonstationary spatial structure using simulated annealing},
  author={Iovleff, Serge and Perrin, Olivier},
  journal={Journal of Computational and Graphical Statistics},
  volume={13},
  number={1},
  pages={90--105},
  year={2004},
  publisher={Taylor \& Francis}
}

@article{damian2001bayesian,
  title={Bayesian estimation of semi-parametric non-stationary spatial covariance structures},
  author={Damian, Doris and Sampson, Paul D and Guttorp, Peter},
  journal={Environmetrics: The official journal of the International Environmetrics Society},
  volume={12},
  number={2},
  pages={161--178},
  year={2001},
  publisher={Wiley Online Library}
}

@article{schmidt2003bayesian,
  title={Bayesian inference for non-stationary spatial covariance structure via spatial deformations},
  author={Schmidt, Alexandra M and O'Hagan, Anthony},
  journal={Journal of the Royal Statistical Society Series B: Statistical Methodology},
  volume={65},
  number={3},
  pages={743--758},
  year={2003},
  publisher={Oxford University Press}
}

@article{anderes2008estimating,
  title={Estimating deformations of isotropic Gaussian random fields on the plane},
  author={Anderes, Ethan B and Stein, Michael L},
  journal={The Annals of Statistics},
  volume={36},
  number={2},
  pages={719--741},
  year={2008}
}

@inproceedings{damianou2013deep,
  title={Deep gaussian processes},
  author={Damianou, Andreas and Lawrence, Neil D},
  booktitle={Artificial intelligence and statistics},
  pages={207--215},
  year={2013},
  organization={PMLR}
}

@article{dunlop2018deep,
  title={How deep are deep Gaussian processes?},
  author={Dunlop, Matthew M and Girolami, Mark A and Stuart, Andrew M and Teckentrup, Aretha L},
  journal={Journal of Machine Learning Research},
  volume={19},
  number={54},
  pages={1--46},
  year={2018}
}

@article{perrin1999identifiability,
  title={Identifiability for non-stationary spatial structure},
  author={Perrin, Olivier and Meiring, Wendy},
  journal={Journal of Applied Probability},
  volume={36},
  number={4},
  pages={1244--1250},
  year={1999},
  publisher={Cambridge University Press}
}

@article{fouedjio2015estimation,
  title={Estimation of space deformation model for non-stationary random functions},
  author={Fouedjio, Francky and Desassis, Nicolas and Romary, Thomas},
  journal={Spatial Statistics},
  volume={13},
  pages={45--61},
  year={2015},
  publisher={Elsevier}
}

@inproceedings{gu2025identification,
  title={Identification of Latent Invariant Surface Quality Patterns via Spatial Stochastic Process Deformation},
  author={Gu, Minghao and Huang, Qiang},
  booktitle={2025 IEEE 21st International Conference on Automation Science and Engineering (CASE)},
  pages={1107--1112},
  year={2025},
  organization={IEEE}
}

@article{hamilton1979inverse,
  title={The inverse function theorem of Nash and Moser},
  author={Hamilton, Richard S},
  journal={Nonlinear and global analysis},
  volume={1},
  pages={139},
  year={1979}
}

@book{brin2002introduction,
  title={Introduction to dynamical systems},
  author={Brin, Michael and Stuck, Garrett},
  year={2002},
  publisher={Cambridge university press}
}

@book{johnson2016handbook,
  title={Handbook of fluid dynamics},
  author={Johnson, Richard W},
  year={2016},
  publisher={CRC press}
}

@article{beg2005computing,
  title={Computing large deformation metric mappings via geodesic flows of diffeomorphisms},
  author={Beg, M Faisal and Miller, Michael I and Trouv{\'e}, Alain and Younes, Laurent},
  journal={International journal of computer vision},
  volume={61},
  number={2},
  pages={139--157},
  year={2005},
  publisher={Springer}
}

@book{fisher1966design,
  title={The design of experiments},
  author={Fisher, Sir Ronald Aylmer and Fisher, Ronald A},
  volume={21},
  year={1966},
  publisher={Springer}
}

@book{higham2008functions,
  title={Functions of matrices: theory and computation},
  author={Higham, Nicholas J},
  year={2008},
  publisher={SIAM}
}

@article{al2010new,
  title={A new scaling and squaring algorithm for the matrix exponential},
  author={Al-Mohy, Awad H and Higham, Nicholas J},
  journal={SIAM Journal on Matrix Analysis and Applications},
  volume={31},
  number={3},
  pages={970--989},
  year={2010},
  publisher={SIAM}
}

@article{hernandez2018newton,
  title={Newton-Krylov PDE-constrained LDDMM in the space of band-limited vector fields},
  author={Hernandez, Monica},
  journal={arXiv preprint arXiv:1807.05117},
  year={2018}
}

@phdthesis{polzin2018large,
  title={Large deformation diffeomorphic metric mappings: Theory, numerics, and applications},
  author={Polzin, Thomas},
  year={2018},
  school={Zentrale Hochschulbibliothek L{\"u}beck}
}

@book{bishop2006pattern,
  title={Pattern recognition and machine learning},
  author={Bishop, Christopher M and Nasrabadi, Nasser M},
  volume={4},
  number={4},
  year={2006},
  publisher={Springer}
}

@article{wan2015mod11a2,
  title={MOD11A2 MODIS/Terra land surface temperature/emissivity 8-day L3 global 1km SIN grid V006},
  author={Wan, Zhengming and Hook, Simon and Hulley, Glynn},
  journal={NASA EOSDIS Land Processes Distributed Active Archive Center},
  year={2015}
}

@article{nasa2013shuttle,
  title={Shuttle Radar Topography Mission Global 1 Arc Second},
  author={NASA JPL, NASA},
  journal={NASA EOSDIS Land Processes Distributed Active Archive Center},
  year={2013}
}

@article{gorelick2017google,
  title={Google Earth Engine: Planetary-scale geospatial analysis for everyone},
  author={Gorelick, Noel and Hancher, Matt and Dixon, Mike and Ilyushchenko, Simon and Thau, David and Moore, Rebecca},
  journal={Remote sensing of Environment},
  volume={202},
  pages={18--27},
  year={2017},
  publisher={Elsevier}
}

@book{jacobson2013lie,
  title={Lie algebras},
  author={Jacobson, Nathan},
  year={2013},
  publisher={Courier Corporation}
}

@book{hartman2002ordinary,
  title={Ordinary differential equations},
  author={Hartman, Philip},
  year={2002},
  publisher={SIAM}
}

@incollection{lee2003smooth,
  title={Smooth manifolds},
  author={Lee, John M},
  booktitle={Introduction to smooth manifolds},
  pages={1--29},
  year={2003},
  publisher={Springer}
}

@article{petersen2008matrix,
  title={The matrix cookbook},
  author={Petersen, Kaare Brandt and Pedersen, Michael Syskind and others},
  journal={Technical University of Denmark},
  volume={7},
  number={15},
  pages={510},
  year={2008}
}

\newpage

\appendix

\section{Proof of Lemma~\ref{lemma_commute}} \label{appendix_1}

\begin{proof}
    We go through two parts to demonstrate the equivalence between the macroscopic path independence (Assumption~\ref{assumption_path_independence}) and the microscopic commutativity of velocity fields. 

    \textbf{Part 1: Necessity}
    
    By Assumption~\ref{assumption_path_independence}, the relative deformations induced by isolated channels $m$ and $n$ commute: 
    $$
        h_k^m \circ h_k^n = h_k^n \circ h_k^m, 
    $$
    and this holds for arbitrary effective covariate changes $\Delta \tau_k^m = \tau_k^m - \tau_0^m \in \mathcal{T}^m - \tau_0^m$.
    We denote the relative deformations induced by changing $s$ on covariate channel $\tau^m$ as $\phi_s^m = \exp(sV_m)$, and similarly for $\phi_t^n = \exp(tV_n)$. 
    Consequently, we also have their commutativity: 
    \begin{equation} \label{eq_commute}
        \phi_s^m \circ \phi_t^n = \phi_t^n \circ \phi_s^m. 
    \end{equation}
    
    The Lie bracket measures the infinitesimal failure of the flows of two vector fields to commute \citep{jacobson2013lie}. 
    Therefore, we evaluate, on a point $\boldsymbol{w} \in \mathcal{W}$, the commutator of the flows along a closed loop $(0,0) \rightarrow (s, 0) \rightarrow (s, t) \rightarrow (0, t) \rightarrow (0, 0)$, with changes $s, t \rightarrow 0$. 
    By definition, going along this closed loop induces the following equation: 
    $$
        \phi_{-t}^n \circ \phi_{-s}^m \circ \phi_t^n \circ \phi_s^m (\boldsymbol{w}) = \boldsymbol{w} + st [V_m, V_n](\boldsymbol{w}) + \mathcal{O}(s^2t + st^2). 
    $$
    Because of Eq.~(\ref{eq_commute}), the left-hand side of the equation above directly reduces to the identity mapping on $\boldsymbol{w}$ through changing the order of the terms. 
    That is, the equation reduces to:
    $$
        st [V_m, V_n](\boldsymbol{w}) + \mathcal{O}(s^2t + st^2) = 0. 
    $$
    As $s,t \neq 0$, and the equation should apply to all locations $\boldsymbol{w} \in \mathcal{W}$, this enforces $[V_m, V_n] = 0$ for any two distinct covariate channels $m \neq n$. 

    \textbf{Part 2: Sufficiency}

    Conversely, if $[V_m, V_n] = 0$, the fundamental theorem of Lie derivatives guarantees that their corresponding lows commute for all integrations \citep{jacobson2013lie}. 
    That is, the relation in Eq.~(\ref{eq_commute}) holds for all finite shifts $s$ and $t$, thereby satisfying Assumption~\ref{assumption_path_independence}. 
\end{proof}

\section{Proof of Proposition~\ref{proposition_linear}} \label{appendix_2}

\begin{proof}
    By definition, the total relative spatial deformation induced by $\boldsymbol{\tau}_k$, the simultaneous shift of $p$ covariate channels from $\boldsymbol{\tau}_0$, is generated by Eq.~(\ref{eq_composition}). 
    Specifically, for any location $\boldsymbol{w} \in \mathcal{W}$, this is given by: 
    $$
        h_k(\boldsymbol{w}) = \exp \left( \sum_{m = 1}^p \Delta \tau_k^m V_m \right) (\boldsymbol{w}). 
    $$

    To evaluate the exponential of a sum, we must rely on the BCH formula stated in Eq.~(\ref{eq_bch}) \citep{jacobson2013lie}. 
    According to Lemma~\ref{lemma_commute}, the DoE path independence grants us the guarantees that the base velocity fields commute: $[V_m, V_n]=0$ for all $m \neq n$. 
    Because the Lie bracket operation is bilinear, the scaled vector fields also commute: 
    $$
        [\Delta \tau_k^m V_m, \Delta \tau_k^n V_n] = \Delta \tau_k^m \Delta \tau_k^n [V_m, V_n] = 0. 
    $$

    All first-order Lie brackets vanish as above, and therefore, every subsequent higher-order nested Lie brackets in Eq.~(\ref{eq_bch}) also vanish. 
    By induction, the exact factorization extends to the sum of all $p$ channels, and consequently, the infinite series truncates to: 
    $$
    \begin{aligned}
        \exp \left( \sum_{m = 1}^p \Delta \tau_k^m V_m \right) (\boldsymbol{w}) & = \left[ \exp(\Delta \tau_k^p V_p) \circ \cdots \circ \exp(\Delta \tau_k^1 V_1) \right] (\boldsymbol{w}) \\
        & = \left[ \bigcirc_{m = 1}^p \exp (\Delta \tau_k^m V_m) \right] (\boldsymbol{w}). 
    \end{aligned}
    $$
    Substituting the isolated spatial deformations defined in Eq.~(\ref{eq_linear}) yields: 
    $$
        h_k(\boldsymbol{w}) = \left( h_k^p \circ \cdots \circ h_k^1 \right) (\boldsymbol{w}). 
    $$
    
    Finally, the total covariate-driven spatial deformation $f_k$ acting on the original spatial domain $\mathcal{S}$ is clearly defined as the application of this total relative deformation to the baseline $f_0$. 
    Specifically, for any location $\boldsymbol{s} \in \mathcal{S}$: 
    $$
    \begin{aligned}
        f_k(\boldsymbol{s}) & = h_k \left( f_0(\boldsymbol{s}) \right) \\
        & = \left( h_k^p \circ \cdots \circ h_k^1 \circ f_0 \right)(\boldsymbol{s}). 
    \end{aligned}
    $$
    
    This concludes the proof, demonstrating that the complex deformation driven by multiple covariates can be exactly computed via sequential isolated deformations.
\end{proof}

\section{Proof of Proposition~\ref{proposition_glm}} \label{appendix_3}

\begin{proof}
    The only distinction to the derivation of Proposition~\ref{proposition_linear} is how the properties of the linking functions $g_m(\cdot)$ are related to the constraints of deformations $f_k(\cdot)$ we gave in Sec.\ref{sec_setup}. 
    Specifically, we need $f_k(\cdot)$ to be orientation-preserving diffeomorphisms. 
    
    Because the composition of diffeomorphisms is surely diffeomorphic, the proof of composite validity reduces to the proof of the relative deformation $h_k^m (\cdot) = \exp \left( g_m(\Delta \tau_k^m) V_m \right)(\cdot)$ induced by any isolated channel. 
    By the inverse function theorem \citep{hamilton1979inverse}, $h_k^m (\cdot)$ being an orientation-preserving diffeomorphism requires the Jacobian determinant to be strictly positive everywhere: $\det \left( \nabla h_k^m(\boldsymbol{w}) \right), \forall \boldsymbol{w} \in \mathcal{W}$. 

    Applying the chain rule on Eq.~(\ref{eq_glm}), the Jacobian of this mapping is (note that covariates can be spatial-varying as stated in Sec .~\ref{sec_setup}): 
    \begin{align}
        \begin{split}
            \label{eq_gradient}
            \nabla h_k^m(\boldsymbol{w}) & = \nabla \left[ \exp \left( g_m(\Delta \tau_k^m(\boldsymbol{w})) V_m \right) (\boldsymbol{w}) \right] \\
            & = \nabla \exp (tV_m) \vert_{t = g_m(\Delta \tau_k^m(\boldsymbol{w}))} + V_m(h_k^m(\boldsymbol{w})) \otimes \left[ \nabla g_m(\Delta \tau_k^m(\boldsymbol{w})) \right]^\top,
        \end{split}
    \end{align}
    in which $V_m(h_k^m(\boldsymbol{w}))$ is the velocity vector at a deformed location $h_k^m(\boldsymbol{w})$, and
    $$
        \nabla g_m(\Delta \tau_k^m(\boldsymbol{w})) = g_m' (\Delta \tau_k^m(\boldsymbol{w})) \nabla \left( \Delta \tau_k^m(\boldsymbol{w}) \right). 
    $$

    The deformation $h_k^m (\cdot) = \exp \left( g_m(\Delta \tau_k^m) V_m \right)(\cdot)$ can be viewed as the composition of two parts: a flow $\exp (tV_m)$ and a spatial-varying field $g_m(\Delta \tau(\cdot))$. 
    Correspondingly, the first term of Eq.~(\ref{eq_gradient}) is the standard Jacobian of the mapping as if $t$ is a constant and takes the value $t = g_m(\Delta \tau_k^m(\boldsymbol{w}))$. 
    By the Picard–Lindel\"of theorem \citep{hartman2002ordinary} and what directly follows \citep{lee2003smooth}, the flow of any smooth velocity field generates a one-parameter group of diffeomorphisms. 
    That is, for any constant value of $t$, $\exp (tV_m)$ is bijective (which also implicitly supports the derivation of Proposition~\ref{proposition_linear}), thereby:
    $$
        \det \left[ \nabla \exp (tV_m) \right] > 0.
    $$
    Consequently, the first term of Eq.~(\ref{eq_gradient}) does not cause any danger of violating the constraints.
    
    Conversely, the threat comes from the second term of Eq.~(\ref{eq_gradient}). 
    The analysis of the second-order could be complex, but fortunately, we can simplify Eq.~(\ref{eq_gradient}) with the matrix determinant lemma $\det \left( A + \boldsymbol{u} \boldsymbol{v}^\top \right) = \det A \left( 1 + \boldsymbol{v}^\top A^{-1} \boldsymbol{u} \right)$ \citep{petersen2008matrix}: 
    \begin{equation}
        \label{eq_determinant}
        \nabla h_k^m(\boldsymbol{w}) = \det \left( A \right) \cdot \left( 1 + \left[ \nabla g_m(\Delta \tau_k^m(\boldsymbol{w})) \right]^\top A^{-1} V_m(h_k^m(\boldsymbol{w})) \right),
    \end{equation}
    where $A = \nabla (tV_m) \vert_{t = g_m(\Delta \tau_k^m(\boldsymbol{w}))}$ for simplicity. 

    Notably, if we view $h_k^m = \exp(g_m(\Delta \tau_k^m) V_m)$ as a static flow driven by the velocity field $V_m$, then $A$ is exactly its Jacobian. 
    By the very fundamental fact that a velocity field is invariant under its own pushforward \citep{lee2003smooth}, mathematically we have: 
    $$
    \begin{aligned}
        A V_m(\boldsymbol{w}) & = V_m(h_k^m(\boldsymbol{w})) \\
        V_m(\boldsymbol{w}) & = A^{-1} V_m(h_k^m(\boldsymbol{w})). 
    \end{aligned}
    $$

    Substitute this into Eq.~(\ref{eq_determinant}): 
    $$
        \nabla h_k^m(\boldsymbol{w}) = \det \left( A \right) \cdot \left( 1 + \left[ \nabla g_m(\Delta \tau_k^m(\boldsymbol{w})) \right]^\top V_m(\boldsymbol{w}) \right). 
    $$

    Because it has been shown that $\det(A)>0$, the entire Jacobian determinant is positive if and only if:
    \begin{equation}
        \label{eq_condition}
        1 + \left[ \nabla g_m(\Delta \tau_k^m(\boldsymbol{w})) \right]^\top V_m(\boldsymbol{w}) > 0. 
    \end{equation}

    As $\left[ \nabla g_m(\Delta \tau_k^m(\boldsymbol{w})) \right]^\top V_m(\boldsymbol{w}) = \langle \nabla g_m(\Delta \tau_k^m(\boldsymbol{w})), V_m(\boldsymbol{w}) \rangle$, it is by definition the directional derivative of the scalar function field $g_m(\Delta \tau_k^m(\cdot))$ in the direction of the vector field $V_m(\cdot)$ at location $\boldsymbol{w}$.
    We consider the integral curve $\gamma(x)$, which is the solution of an ODE:
    $$
        \frac{d \gamma(x)}{dx} = V_m(\gamma(x)). 
    $$
    Based on the parameterization with this integral curve, the directional derivative can be alternatively written as: 
    $$
        \frac{d}{dx} g_m(\Delta \tau_k^m(\gamma(x))),
    $$
    on which we apply the chain rule: 
    $$
    \begin{aligned}
        \frac{d}{dx} g_m(\Delta \tau_k^m(\gamma(x))) & = \nabla g_m(\Delta \tau_k^m(\gamma(x))) \cdot \frac{d \gamma(x)}{dx} \\
        & = \nabla g_m(\Delta \tau_k^m(\gamma(x))) \cdot V_m(\gamma(x)).
    \end{aligned}
    $$
    The last step is because we have defined $\gamma(s)$ to be the solution of the ODE parameterized by $V_m$ as above. 

    Expanding the gradient using the chain rule yields: 
    $$
        \nabla g_m(\Delta \tau_k^m(\gamma(x))) = {g_k^m}^\prime (\Delta \tau_k^m(\gamma(x))) \nabla \left( \Delta \tau_k^m(\gamma(x)) \right).
    $$
    Substituting this expanded form of gradient back into Eq.~(\ref{eq_condition}): 
    $$
        1 + g_m^\prime(\Delta \tau_k^m(\gamma(x))) \left[ \nabla \Delta \tau_k^m(\gamma(x)) \cdot V_m(\gamma(x)) \right] > 0.
    $$

    Let the integral curve $\gamma(s)$ be a trajectory that passes through $\boldsymbol{w}$, and this inequality explicitly reveals the mechanism of potential folding. 
    By Cauchy-Schwarz inequality, and our $ L_m$-Lipschitz condition: 
    $$
    \begin{aligned}
        \vert \nabla \Delta \tau_k^m(\gamma(x)) \cdot V_m(\gamma(x)) \vert & \leq \Vert \nabla \Delta \tau_k^m(\gamma(x)) \Vert_2  \cdot \Vert V_m(\gamma(x)) \Vert_2 \\
        & \leq \Vert \nabla \Delta \tau_k^m(\gamma(x)) \Vert_2 \cdot 1 \\
        & \leq L_m. 
    \end{aligned}
    $$
    Because we have $\vert g_m' \vert \leq \frac{1}{L_M}$, these explicitly lead to: 
    \begin{gather*}
        -L_M \leq \nabla \Delta \tau_k^m(\gamma(x)) \cdot V_m(\gamma(x)) \leq L_M \\
        -1 \leq g_m^\prime(\Delta \tau_k^m(\gamma(x))) \left[ \nabla \Delta \tau_k^m(\gamma(x)) \cdot V_m(\gamma(x)) \right] \leq 1 \\
        0 \leq 1 + \leq g_m^\prime(\Delta \tau_k^m(\gamma(x))) \left[ \nabla \Delta \tau_k^m(\gamma(x)) \cdot V_m(\gamma(x)) \right] \leq 2,
    \end{gather*}
    which directly proves Eq.~(\ref{eq_condition}). 

    The result holds for all spatial locations $\boldsymbol{w} \in \mathcal{W}$, and their unique one-dimensional trajectory $\gamma(x)$ in the space. 
    Essentially, this traverses all particles in the domain and carefully checks that their paths are free of folding \citep{lee2003smooth}. 
    This concludes the proof. 
\end{proof}

\end{document}